\def\year{2021}\relax
\documentclass[letterpaper]{article} 
\usepackage{aaai21}  
\usepackage{times}  
\usepackage{helvet} 
\usepackage{courier}  
\usepackage[hyphens]{url}  
\usepackage{graphicx} 
\usepackage{natbib}  
\usepackage{caption} 
\usepackage{array}
\usepackage{tabularx}
\usepackage{microtype}
\usepackage{wrapfig}
\usepackage{graphicx}
\usepackage{subfigure}
\usepackage{bm}
\usepackage{comment}
\usepackage{amsmath,amssymb,amsthm,mathtools}
\usepackage[ruled,vlined,linesnumbered]{algorithm2e}
\usepackage{setspace}

\renewcommand{\mid}{\,:\,}

\renewcommand{\mid}{\,:\,}

\newcommand{\Prob}{\mbox{\rm Pr}}

\newcommand{\Mlin}{M(\lambda)^{-1}}

\newcommand{\Atl}{A^{\iota}_{{\bf x}_t}}
\newcommand{\Atinl}{({A^{\iota}_{{\bf x}_t}})^{-1}}
\newcommand{\Atauinl}{({A^{\lambda}_{{\bf x}_{\tau_e}}})^{-1}}
\newcommand{\Lamdatauinl}{({\Lambda^{\iota}_{{\bf x}_{\tau_e}}})^{-1}}

\newcommand{\thetahat}{\widehat{\theta}_t}

\theoremstyle{plain}
\newtheorem{theorem}     {Theorem}
\newtheorem{lemma}       {Lemma}
\newtheorem{proposition} {Proposition}
\newtheorem{corollary}   {Corollary}

\newtheorem{assumption} {Assumption}

\usepackage[usenames,dvipsnames]{xcolor}
\usepackage{ifthen}

\newcommand{\argmax}{\operatornamewithlimits{argmax}}
\newcommand{\argmin}{\operatornamewithlimits{argmin}}

\newcommand{\E}{\mathbb{E}}

\newcommand{\R}{\mathbb{R}}

\newcommand{\cE}{\mathcal{E}}
\newcommand{\cF}{\mathcal{F}}

\newcommand{\cX}{\mathcal{X}}

\newcommand{\compilefullversion}{false}
\ifthenelse{\equal{\compilefullversion}{true}}{%
	\newcommand{\OnlyInFull}[1]{}
	\newcommand{\OnlyInShort}[1]{#1}
}{%
	\newcommand{\OnlyInFull}[1]{#1}%
	\newcommand{\OnlyInShort}[1]{}%
}%

\newcommand{\compilehidecomments}{false}
\ifthenelse{ \equal{\compilehidecomments}{true} }{%
	\newcommand{\wei}[1]{}
	\newcommand{\yuko}[1]{}
	\newcommand{\yihan}[1]{}
}{
	\newcommand{\wei}[1]{{\color{blue}  [\text{Wei:} #1]}}
	\newcommand{\yuko}[1]{{\color{purple} [\text{Yuko:} #1]}}
	\newcommand{\yihan}[1]{{\color{teal} [\text{Yihan:} #1]}}
}

\title{Combinatorial Pure Exploration with Full-Bandit or \\ Partial Linear Feedback}
\author{

    Yihan Du,\textsuperscript{\rm 1}\thanks{The first two authors have equal contributions.} 
    Yuko Kuroki,\textsuperscript{\rm 2}\footnotemark[1] 
    Wei Chen\textsuperscript{\rm 3}\\
}
\affiliations{

    \textsuperscript{\rm 1}IIIS, Tsinghua University,
    \textsuperscript{\rm 2}The University of Tokyo, RIKEN,
    \textsuperscript{\rm 3}Microsoft Research

    duyh18@mails.tsinghua.edu.cn, ykuroki@ms.k.u-tokyo.ac.jp, weic@microsoft.com

}

\begin{document}

 \abovedisplayskip=3.5pt
 \belowdisplayskip=3.5pt
\setlength\textfloatsep{14pt}

\maketitle


\begin{abstract}
In this paper, we first study the problem of combinatorial pure exploration with full-bandit feedback (CPE-BL), where a learner is given a combinatorial action space $\mathcal{X} \subseteq \{0,1\}^d$, and in each round the learner pulls an action $x \in \mathcal{X}$ and receives a random reward with expectation $x^{\top} \theta$, with $\theta \in \R^d$ a latent and unknown environment vector. 
The objective is to identify the optimal action with the highest expected reward, using as few samples as possible. For CPE-BL, we design the first {\em polynomial-time adaptive} algorithm, whose sample complexity matches the lower bound (within a logarithmic factor) for a family of instances and has a light dependence of $\Delta_{\min}$ (the smallest gap between the optimal action and sub-optimal actions).
Furthermore, we propose a novel generalization of CPE-BL with  flexible feedback structures, called combinatorial pure exploration with partial linear feedback (CPE-PL), which encompasses several families of sub-problems including full-bandit feedback, semi-bandit feedback, partial feedback and nonlinear reward functions. In CPE-PL, each pull of action $x$ reports a random feedback vector with expectation of $M_{x} \theta $, where $M_x \in \mathbb{R}^{m_x \times d}$ is a transformation matrix for $x$, and gains a random (possibly nonlinear) reward related to $x$. 
For CPE-PL, we develop the first {\em polynomial-time} algorithm, which simultaneously addresses limited feedback, general reward function and combinatorial action space (e.g., matroids, matchings and $s$-$t$ paths), and provide its sample complexity analysis.
Our empirical evaluation demonstrates that our algorithms run orders of magnitude faster than the existing ones, and our CPE-BL algorithm is robust across different $\Delta_{\min}$ settings while our CPE-PL algorithm is the first one returning correct answers for nonlinear reward functions.
\end{abstract}
\section{Introduction}\label{sec:intro}

The problem of \emph{best arm identification} (BAI) is the pure-exploration framework in stochastic multi-armed bandits.
In BAI, at each step a learner chooses an arm and observes its reward sampled from an unknown distribution, with the goal of returning the best arm with the highest expected reward using as few exploration steps as possible.
This problem abstracts a decision making model in the face of uncertainty with a wide range of applications, and has received much attentions in the literature~\cite{Even2006,Audibert2010, Chen2015,kaufmann2016}.

In many application domains, possible actions have a certain combinatorial structure.
For example, each action may be a size-$k$ subset of keywords in online advertisements~\cite{Rusmevichientong2006}, or an assignment between workers and tasks in crowdsourcing~\cite{LinTian2014}, or a spanning tree in communication networks~\cite{Huang2008}.
To deal with such a combinaotrial action space, the model of \emph{combinatorial pure exploration  of multi-armed bandits} (CPE-MB) was first proposed by Chen et al.~\shortcite{Chen2014}.
In this model, there are $d$ base arms, each of which is associated with an unknown reward distribution, and a collection of {\em super arms}, each of which is a subset of base arms.
A learner plays a base arm at each step and observes its random reward, with the goal of identifying the best super arm that maximizes the sum of expected rewards at the end of exploration.
CPE-MB generalizes the classical BAI problem~\cite{ kalyanakrishnan2010,kalyanakrishnan2012,bubeck2013}.

However, many real-world scenarios may not fit into CPE-MB.
In particular, CPE-MB assumes that the learner can directly play each base arm and observe its outcome,
but this might not be allowed due to system constraints or privacy issues.
Only a few studies avoid such an assumption.
Kuroki et al.~\shortcite{Kuroki+19} studied the \emph{combinatorial pure exploration with full-bandit linear feedback (CPE-BL)},
in which the learner pulls a super arm (rather than base arm) and only observes the sum of rewards from the involved base arms. 
They designed an efficient algorithm for CPE-BL, but the algorithm is nonadaptive and its sample complexity heavily depends on 
	the smallest gap between the best and the other super arms (denoted by $\Delta_{\min}$).
Rejwan and Mansour~\shortcite{Idan2019} also designed an efficient algorithm with an adaptive Successive-Accept-Reject algorithm, but the
	algorithm only works for the top-$k$ case of CPE-BL, which we show can be simply reduced to previous CPE-MB (see Appendix~\ref{apx:clucb}\OnlyInShort{ in the full version}).


Note that CPE-BL can be regarded as an instance of the \emph{best arm identification in linear bandits} (BAI-LB), which has received increasing attention recently~\cite{Soare2014,Tao2018,Fiez2019}. 
However, none of the existing algorithms for BAI-LB can  efficiently solve CPE-BL, because their running times have polynomial dependence on the size of action space, which is exponential in the combinatorial setting.

\begin{table*}[t]
	\centering
	\caption{Comparison between our results and state-of-the-art results for CPE-BL(PL). ``General'' represents that the algorithm works for any combinatorial structure. $\tilde{O}(\cdot)$ only omits $\log \log$ factors. Main notations is defined in Section~\ref{sec:problem_def}. 
	} 
	\label{table:results}
	\renewcommand\arraystretch{1.8}
	\scalebox{0.75}{
	\begin{tabular}{|c|c|c|c|c|c|}
		\hline
		Algorithm&Sample complexity&Case&Problem Type&Strategy&Time\\
		\hline
		{\bf 
		\textsf{GCB-PE} (ours, Thm.~\ref{thm:GCB_ub})} &$O \big( \frac{ |\sigma| \beta_{\sigma}^2 L_p^2}{\Delta_{\textup{min}}^2} \log \frac{ \beta_{\sigma}^2 L_p^2}{\Delta_{\textup{min}}^2\delta} \big)$& General &CPE-PL&Static&$\mathrm{Poly}(d)$\\
		\hline
		{\bf  
		\textsf{PolyALBA} (ours, Thm.~\ref{thm:CPE-BL})} &$\tilde{O} \big(\sum_{i=2}^{\lfloor \frac{d}{2} \rfloor} \frac{1}{\Delta_i^2} \log \frac{|\mathcal{X}|}{\delta}    +\frac{d^2 m \xi_{\max}({\widetilde{M}({\lambda})}^{-1})}{ \Delta^2_{d+1}} \log \frac{|\mathcal{X}|}{\delta} \big)$&General &CPE-BL& Adaptive  &$\mathrm{Poly}(d)$\\
		\hline
		\textsf{ICB}~\cite{Kuroki+19}&$\tilde{O}\big(\frac{d  \xi_{\max}({M({\lambda})}^{-1})\rho(\lambda)}{\Delta^2_{\min}} \log \frac{d  \xi_{\max}({M({\lambda})}^{-1}) \rho(\lambda)}{\Delta^2_{\min}\delta}\big)$ &General &CPE-BL& Static&$\mathrm{Poly}(d)$\\
		\hline
		\textsf{CSAR}~\cite{Idan2019} &$\tilde{O}\big(\sum^{d}_{i=2}\frac{1}{\tilde{\Delta}_i^2} \log\frac{d}{\delta}\big)$&Top-$k$ &CPE-BL&  Adaptive &$\mathrm{Poly}(d)$ \\
		\hline
	    \textsf{ALBA}~\cite{Tao2018}&$\tilde{O}\big(\sum^{d}_{i=2}\frac{1}{\Delta_i^2}(\log \delta^{-1} + \log |\mathcal{X}|)\big)$&$\cX \subseteq \mathbb{R}^d$&BAI-LB& Adaptive & $\Omega(|\cX|)$\\
		\hline
		\textsf{RAGE}~\cite{Fiez2019}&$O \big(\sum^{\left \lfloor \log_2(4/\Delta_{\min})  \right \rfloor}_{t=1} 2(2^t)^2\tilde{\rho}(\mathcal{Y}(S_t)) \log(t^2|\mathcal{X}|^2/\delta) \big)$&$\cX \subseteq \mathbb{R}^d$&BAI-LB& Adaptive& $\Omega(|\cX|)$\\
		\hline
		\textsf{LinGame(-C)} \cite{Degenne+2020}&$\mathop{\lim \sup}_{\delta \rightarrow 0} \frac{\mathbb{E}_{\theta}[\tau_{\delta}]}{\log(1/ \delta)} \leq \min_{\lambda \in \triangle(\cX)}\max_{x\in \mathcal{\cX} \setminus\{x^*\}} \frac{2||x^*-x||_{M(\lambda)^{-1}}^2}{((x^*-x)^{\top}\theta)^2} $&$\cX \subseteq \mathbb{R}^d$&BAI-LB& Adaptive& $\Omega(|\cX|)$ \\
		\hline
		\textsf{Peace} \cite{Katz+2020}&$O\big( (\min_{\lambda \in \triangle(\cX)}\max_{x\in \mathcal{\cX} \setminus\{x^*\}} \frac{||x^*-x||_{M(\lambda)^{-1}}^2}{((x^*-x)^{\top}\theta)^2} + \gamma^* )  \log(1/ \delta)\big)$&$\cX \subseteq \mathbb{R}^d$&BAI-LB& Adaptive& $\Omega(|\cX|)$ \\
		\hline
		\textsf{LT\&S} \cite{jedra2020optimal}&$\mathop{\lim \sup}_{\delta \rightarrow 0} \frac{\mathbb{E}_{\theta}[\tau_{\delta}]}{\log(1/ \delta)} \leq \min_{\lambda \in \triangle(\cX)}\max_{x\in \mathcal{\cX} \setminus\{x^*\}} \frac{||x^*-x||_{M(\lambda)^{-1}}^2}{((x^*-x)^{\top}\theta)^2} $&$\cX \subseteq \mathbb{R}^d$&BAI-LB& Adaptive& $\Omega(|\cX|)$ \\
		\hline
		Lower Bound \cite{Fiez2019}&$\mathbb{E}_{\theta}[\tau_{\delta}] \geq \min_{\lambda \in \triangle(\cX)}\max_{x\in \mathcal{\cX} \setminus\{x^*\}} \frac{||x^*-x||_{M(\lambda)^{-1}}^2}{((x^*-x)^{\top}\theta)^2} \log(1/2.4 \delta)$&$\cX \subseteq \mathbb{R}^d$&BAI-LB& - & - \\
		\hline
	\end{tabular}
	}
\end{table*}

In this paper, we provide the first algorithm solving CPE-BL that simultaneous achieves the following properties:
	(a) polynomial-time complexity, (b) adaptive sampling, such that the sample complexity is not heavily dependent on $\Delta_{\min}$; 
	(c) general combinatorial constraints, and
	(d) nearly optimal sample complexity for some family of instances.

Next, we propose a more general setting, \emph{combinatorial pure exploration with partial linear feedback (CPE-PL)}, which simultaneously models limited feedback, general (possibly nonlinear) reward and combinatorial action space.  In CPE-PL, given a combinatorial action space $\cX \subseteq \{0,1\}^d$, where each dimension corresponds to a base arm and each action $x\in \cX$ can also be viewed as a super arm that contains those dimensions with coordinate $1$.
At each step the learner chooses an action (super arm) $x_t \in \cX$ to play and observes a random partial linear feedback with expectation of $M_{x_t} \theta$, 
	where $M_{x_t}$ is a transformation matrix for $x_t$ and $\theta \in \mathbb{R}^d$ is an unknown environment vector. 
The learner also gains a random (possibly nonlinear) reward related to $x_t$ and $\theta$, which may not be a part of the  feedback and thus may not be directly observed. 
Given a confidence level $\delta$, the objective is to identify the optimal action with the maximum expected reward with probability at least $1-\delta$, using as few samples as possible.
CPE-PL framework includes CPE-BL as its important sub-problem.
In CPE-BL, the learner observes full-bandit feedback (i.e. $M_x = x^{\top}$) and gains linear reward (with expectation of 
$x^{\top} \theta$) after each play. 

The model of CPE-PL appears in many practical scenarios. For example, in online ranking~\cite{Sougata-Ambuj2017}, a company recommends their products to users by presenting rankings of entire items, and wants to find the best ranking with limited feedback on the top-ranked item due to user burden constraints and privacy concerns. 
In crowdsourcing~\cite{LinTian2014}, an employer assigns crowdworkers to tasks according to the worker-task performance, and wants to find the best matching with limited feedback on a small subset of the completed tasks, owing to the burden of entire feedback and privacy issues (see Section~\ref{sec:PL_application} for detailed  applications).

We remark that, CPE-PL is a novel and general model that encompasses several families of sub-problems across full-bandit feedback, semi-bandit feedback and nonlinear reward function, and it cannot be translated to CPE-BL or BAI-LB. 
For example, when the reward function is $(x^{\top} \theta)/ \|x\|_1$ and $M_x=\textup{diag}(x)$, CPE-PL reduces to a semi-bandit problem with nonlinear reward function, and no existing CPE-BL or BAI-LB algorithm could solve this problem.

Finally, we empirically compare our algorithms with several state-of-the-art CPE-BL and BAI-LB algorithms.
Our result demonstrates that (a) our algorithms run much faster than all others, some of which cannot even finish after days of running;
	(b) For CPE-BL, our adaptive algorithm is much more robust on different $\Delta_{\min}$ settings than the existing nonadaptive algorithm; and
	(c) For CPE-PL, our algorithm is the only one that correctly outputs the optimal action for a nonlinear reward function among all the compared algorithms.

To summarize, our contributions include:
	(a) proposing the first {\em polynomial-time adaptive} algorithm for CPE-BL with general constraints that achieves near optimal sample complexity for some family of instances; and
	(b) proposing the general CPE-PL framework and the first {\em polynomial time} algorithm for CPE-PL with its sample complexity analysis.

Due to the space constraint, full proofs with additional results and discussions are moved to the appendices\OnlyInShort{ in the full version}.

\subsection{State-of-the-art Related Work}
Here we compare with the most related and state-of-the-art works (see Table~\ref{table:results}), and 
	the full discussion and comparison table with notation definitions are included in Appendix~\ref{apx:related_work}\OnlyInShort{ in the full version}.
For CPE-BL, Kuroki et al.~\shortcite{Kuroki+19} propose a polynomial-time but static algorithm \textsf{ICB}, which has a heavy dependence on $\Delta_{\textup{min}}$ in the sample complexity and requires a large number of samples for small-$\Delta_{\textup{min}}$ instances empirically (see Appendix~\ref{apx:top_k_experiments}\OnlyInShort{ in the full version}). Rejwan and Mansour~\shortcite{Idan2019} develop a polynomial-time adaptive algorithm \textsf{CSAR} but it only works for the top-$k$ case, which has a naive reduction to previous CPE-MB (see Appendix~\ref{apx:clucb}\OnlyInShort{ in the full version}). 

For BAI-LB where the action space is often considered small, Tao, Blanco, and Zhou~\shortcite{Tao2018} propose an adaptive algorithm \textsf{ALBA} with a light $\Delta_{\textup{min}}$ dependence. Fiez et al.~\shortcite{Fiez2019} present the first lower bound 
	and a nearly optimal algorithm \textsf{RAGE}. Recently, Katz-Samuels et al.~\shortcite{Katz+2020} also design an improved nearly optimal algorithm \textsf{Peace}, which is built upon the previous \textsf{RAGE}. \citet{Degenne+2020} and \citet{jedra2020optimal} develop asymptotically optimal algorithms, 
	but a fair way to compare their results with other non-asymptotical results is unknown. 
While the existing BAI-LB algorithms achieve satisfactory sample complexity, none of them can efficiently solve CPE-BL with 
	an exponentially large combinatorial action space.
	This paper proposes the first polynomial-time adaptive algorithm for CPE-BL, which is nearly optimal for some family of instances, and the first polynomial-time algorithm for CPE-PL.

\section{Problem Statements}\label{sec:problem_def}
\noindent
{\bf Combinatorial pure exploration with full-bandit linear feedback (CPE-BL). \ }
In CPE-BL, a learner is given $d$ base arms numbered $1,2,\dots,d$. We define $\mathcal{X} \subseteq \{0,1\}^d$ as a collection of subsets of base arms, which satisfies a certain combinatorial structure such as size-$k$ subsets, matroids, paths and matchings. A subset of base arms $x \in \cX$ is called a super arm (or an action). Let $m$ denote the maximum number of base arms that a super arm in $\cX$ contains, i.e. $m=\max_{x \in \cX} \|x \|_1$ $(m \leq d)$.  
There is an unknown environment vector $\theta \in \mathbb{R}^d$ with $\| \theta\|_2 \leq L$.
At each time step $t$, a learner pulls a super arm $x_t$ and receives a random reward (full-bandit feedback)  $y_t=x^{\top} (\theta + \eta_t )$, where $\eta_t$ is a zero-mean noise vector bounded in $[-1,1]^d$ and it is independent among different time step $t$. 
Let $x^* =\argmax_{x \in \mathcal{X}} x^{\top} \theta$ denote the optimal super arm, and we assume that the optimal $x^*$ is unique as previous pure exploration works~\cite{Chen2014,LinTian2014,Fiez2019} do.
Let $\Delta_i$ denote the gap of the expected rewards between $x^*$ and the super arm with the $i$-th largest expected reward.

Given a confidence level $\delta \in (0,1)$, the objective  is to use as few samples as possible to identify the optimal super arm with probability at least $1-\delta$.
This is often called the fixed confidence setting in the bandit literature, and
the number of samples required by the learner is called \emph{sample complexity}.

\OnlyInFull{CPE-BL is a combinatorial adaption of the BAI-LB problem, where the action space $\cX \subseteq \mathbb{R}^d$ but $|\cX|$ is often assumed to be small. In contrast, the combinatorial action space setting of  CPE-BL has a  size exponential to the problem instance parameter, and thus requires additional techniques to address the computational challenges.}

\vspace{1mm}
\noindent
{\bf Combinatorial pure exploration with partial-monitoring linear feedback (CPE-PL). \ }
CPE-PL is a generalization of CPE-BL to partial linear feedback and nonlinear reward functions. In CPE-PL, each super arm $x \in \mathcal{X}$ is associated with a transformation matrix $M_x \in \mathbb{R}^{m_x \times d}$, whose row dimension $m_x$ depends on $x$. 
At each timestep $t$, a learner pulls a super arm $x_t$ and observes a random linear feedback  vector $y_t=M_{x_t} (\theta + \eta_t ) \in \mathbb{R}^{m_{x_t}}$, where $\eta_t$ is the noise vector. 
Meanwhile, the learner gains a random reward with expectation of $\bar{r}(x_t, \theta)$. 
Note that for each pull of super arm $x_t$, the actual expected reward $\bar{r}(x_t, \theta)$ may not be part of the linear feedback vector $y_t$ and thus may not be directly observed by the learner.
Similarly, given a confidence  $\delta \in (0,1)$, the learner aims to use as few samples as possible to identify the optimal super arm with probability at least $1-\delta$.

CPE-PL allows more  flexible feedback structures than CPE-BL or BAI-LB, and encompasses several families of sub-problems including full-bandit feedback, semi-bandit feedback and nonlinear reward functions.
For example, when $M_x=x^{\top} \in  \mathbb{R}^{1 \times d}$ for all $x \in \mathcal{X}$, this model reduces to CPE-BL. When $M_x=\textup{diag}(x) \in \mathbb{R}^{d \times d}$ for all $x \in \mathcal{X}$, this model reduces to combinatorial pure exploration with semi-bandit feedback (see Appendix~\ref{apx:model_example}\OnlyInShort{ in the full version} for illustration examples). 

The regret minimization version of CPE-PL has been studied in \citet{LinTian2014,Sougata-Ambuj2016}. 
In this paper, we study the pure exploration version and inherit the two technical assumptions from \citet{LinTian2014,Sougata-Ambuj2016} in order to design an efficient algorithm.

\begin{assumption}[Lipschitz continuity of the expected reward function] \label{assumption_Lipschitz}
    There exists a constant $L_p$ such that for any $x \in \cX$ and any $\theta_1, \theta_2 \in \mathbb{R}^d$,
    $|\bar{r}(x, \theta_1)-\bar{r}(x, \theta_2)| \leq L_p ||\theta_1-\theta_2||_2$.
\end{assumption}

\begin{assumption}[Global observer set] \label{assumption_global}
    There exists a global observer set $\sigma=\{x_1, x_2, \dots, x_{|\sigma|}\} \subseteq \cX$,  such that the stacked $\sum_{i=1}^{|\sigma|} m_{x_i} \times d$ transformation matrix 
    $M_{\sigma}=(M_{x_1}; M_{x_2}; \dots; M_{x_{|\sigma|}})$ is of full column rank ($rank(M_{\sigma})=d$).
\end{assumption}

 Then, the Moore-Penrose pseudoinverse $M_{\sigma}^+$ satisfies $M_{\sigma}^+ M_{\sigma}=I_d$, where $I_d$ is the $d \times d$ identity matrix.
We justify Assumption~\ref{assumption_global} by the fact that without the existence of global observer set, the learner cannot recover $\theta$ and may not distinguish two different actions.
With Assumption~\ref{assumption_global}, we can systematically construct a global observer set with $|\sigma| \leq d$ by sequentially adding an action that strictly increases the rank of $M_{\sigma}$, until $M_{\sigma}$ reaches the full rank. 
Section~\ref{sec:PL_application} provides more detailed discussion on the global observer set with applications of CPE-PL.

\vspace{1mm}
\noindent
{\bf Notations. \ }
For clarity, we also introduce the following notations.
Let $[d]=\{1,2,\dots, d\}$.
For a vector $x \in \mathbb{R}^d $ and a matrix $B\in \mathbb{R}^{d\times d}$, let $\|x\|_B=\sqrt{x^{\top} Bx}$. 
For a positive definite matrix $B\in \mathbb{R}^{d\times d}$, we use $B^{1/2}$ to denote the unique positive definite matrix whose square is $B$.
For a given family $\mathcal{X}$, we use $\triangle(\mathcal{X})$ to denote the set of probability distributions over $\mathcal{X}$.
For distribution $\lambda \in \triangle(\mathcal{X})$, we define  $\mathrm{supp}(\lambda)=\{ x \mid \lambda(x) > 0\}$, $M(\lambda)=\E_{z \sim \lambda}[z z^{\top}]$
and ${\widetilde{M}({\lambda}})=\sum_{x \in \mathrm{supp}(\lambda)}x x^{\top}$.
We denote the maximum (minimal) eigenvalue of matrix $B$ by $ \xi_{\max}(B)$ ($\xi_{\min}(B)$).
\section{Combinatorial Pure Exploration with Full-bandit Feedback (CPE-BL)}\label{sec:algorithm}



    



    
    


In this section, we propose the first polynomial-time adaptive algorithm  \textsf{PolyALBA} for CPE-BL, and show that its sample complexity matches the lower bound (within a logarithmic factor) for a family of instances.

\subsection{Algorithm Procedure}

\noindent
{\bf ALBA algorithm. \ }
Before stating the main algorithm, we introduce the Adaptive Linear Best Arm (\textsf{ALBA}) algorithm  for BAI-LB~\cite{Tao2018} (see Algorithm~\ref{alg:ALBA} for its description), which is the key subroutine of our proposed method \textsf{PolyALBA}.
First, we describe the randomized least-square estimator defined by Tao et al.~\shortcite{Tao2018}.
Let $y_1, \ldots, y_n$ be $n$ $i.i.d.$ samples following a given distribution $\lambda \in \triangle(\cX)$, and let the corresponding rewards be $r_1, \ldots,r_n$ respectively. Let $b=\sum_{i=1}^n r_i y_i$.
Then, the randomized estimator $\hat{\theta}$ is given by  $\hat{\theta}=A^{-1}b$,
where $A=n M(\lambda) \in \mathbb{R}^{d \times d}$ (recall that $M(\lambda)=\E_{z \sim \lambda}[z z^{\top}]$).
The procedure for computing the estimate for $\theta$ is described in \textsf{VectorEst} (Algorithm~\ref{alg:vectorest}).
\textsf{ALBA} is an elimination-based algorithm, where in round $q$ it 
identifies the top $d/2^q$ arms and discards the remaining arms by means of $\textsf{ElimTil}_p$ (Algorithm~\ref{alg:elimtil}).
Note that \textsf{ALBA}$(S, \delta)$ runs in time polynomial to $|S|$.
However,
since in CPE-BL, $|\cX|$ is exponential to the instance size, it is infeasible to run \textsf{ALBA} with $S=\cX$.
Our main contribution is the nontrivial construction of a polynomial sized $S_1$ to run \textsf{ALBA} with.

\begin{algorithm}[t]
	\caption{\textsf{ALBA}$(S, \delta)$~\cite{Tao2018}}\label{alg:ALBA}
	\SetKwInOut{Input}{Input}
	\SetKwInOut{Output}{Output}
	\Input{Action set $S$ and confidence $\delta$.}
	Initialize $S_1 \leftarrow S$;

	\For{$q \leftarrow  1, \dots, \left \lfloor  \log_2d \right \rfloor$}
	{
		$\delta_q \leftarrow \frac{6}{\pi^2} \frac{\delta}{(q+1)^2}$\;
		$S_{q+1}\leftarrow \textup{ElimTil}_{\left \lfloor  \frac{d}{2^q} \right \rfloor}(S_q, \delta_q)$;

         $q \leftarrow q+1$;
    
	}

    \Output{ $x \in S_{q+1}$}
\end{algorithm}

\begin{algorithm}[t]
	\caption{$\textsf{ElimTil}_p(S,\delta)$ }\label{alg:elimtil}
	\SetKwInOut{Input}{Input}
	\SetKwInOut{Output}{Output}
	\Input{A parameter $p$, arms set $S$ and confidence level $\delta$.}
	Compute $\lambda^*_{S} \leftarrow \min_{\lambda \in \triangle(S)}\max_{x \in S} x^{\top} M(\lambda)^{-1}x$;

	Initialize $S_1 \leftarrow S$, $r \leftarrow 1$;

	\While{$|S_r|>p$}{
		Set $\varepsilon_r \leftarrow 1/2^r$, $\delta_r \leftarrow 6/\pi^2 \cdot \delta/r^2$;

		$\hat{\theta}_r \leftarrow$ \textsf{VectorEst}$(\lambda^*_{S}, c_0 \frac{2+(6+\varepsilon_r/2)d}{(\varepsilon/2)^2}\ln \frac{5|S|}{\delta_r})$;
		
		$x_r \leftarrow \argmax_{x \in S_r} x^{\top}\hat{\theta}_r $;
		
		$S_{r+1} \leftarrow S_r \setminus \{ x \in S_r \ | \ x^{\top} \hat{\theta}_r < x_r^{\top}\hat{\theta}_r-\varepsilon_r \}$;
		
		$r \leftarrow r+1$;
		
	}
	\Output{$S_r$}
\end{algorithm}

\begin{algorithm}[t]
	\caption{\textsf{VectorEst}($\lambda,n$) }\label{alg:vectorest}
	\SetKwInOut{Input}{Input}
	\SetKwInOut{Output}{Output}
	\Input{distribution $\lambda$ and the number of samples $n$}

    Let $y_1 \ldots, y_n$ be the $n$ samples acquired from $\mathrm{supp}(\lambda)$ according to the distribution $\lambda$;
    
    Pull arms $y_1, \ldots, y_n$;
    
    Observe the rewards $r_1,\ldots,r_n$;
    
    $A \leftarrow n \cdot \sum_{x \in \mathrm{supp}(\lambda)}\lambda(x)xx^{\top}$;
    
    $b \leftarrow \sum_{i=1}^n r_i y_i$;

    \Output{The estimate $\hat{\theta} \leftarrow A^{-1}b$}
\end{algorithm}

\begin{algorithm}[t]
	\caption{\textsf{PolyALBA}}
	\label{alg:PolyALBA}
	\SetKwInOut{Input}{Input}
	\SetKwInOut{Output}{Output}
	\Input{confidence level $\delta$, $c_0=\max \{4L^2, 3\}$.}
	Set $q \leftarrow  0$ and $\delta_q \leftarrow \frac{6}{\pi^2} \frac{\delta}{(q+1)^2}$\;
	\label{line:computelambda}
	Compute a distribution $\lambda \leftarrow \lambda^*_{\mathcal{X}_{\sigma}}$ and parameter $\alpha \leftarrow  \sqrt{md/\xi_{\min}(\widetilde{M}(\lambda^*_{\mathcal{X}_{\sigma}}))}$ by Algorithm~\ref{alg:compute_dist}\;
	$r \leftarrow  1$\;
	\While{\textup{true} \label{line:whilebfirstepoch}} 
	{
		Set $\varepsilon_r \leftarrow \frac{1}{2^r}$ and
		$\delta_r \leftarrow \frac{6}{\pi^2} \frac{\delta_q}{r^2}$\;
		
		$ \ell(\varepsilon) \leftarrow \frac{2m + 2 \alpha \sqrt{m} d + 4 \alpha^2 d + \alpha \varepsilon d}{\varepsilon^2}$;
		
		$\hat{\theta}_r \leftarrow \textsf{VectorEst}(\lambda, c_0 \ell(\frac{\varepsilon_r}{2})  \ln(\frac{5|\mathcal{X}|}{\delta_r} ) )$\;
		Select $d+1$ actions $\hat{x}_1, \dots, \hat{x}_d, \hat{x}_{d+1}$ with the highest $d+1$ empirical means $x^{\top} \hat{\theta}_r$ in all $x \in \mathcal{X}$\;
		\If {$\hat{x}_1^{\top} \hat{\theta}_r-\hat{x}_{d+1}^{\top} \hat{\theta}_r > \varepsilon_r$}
		{
			$B_{1} \leftarrow \{\hat{x}_1, \dots, \hat{x}_d\}$\; \label{line:select_S1_start}
		
			$S_{1} \leftarrow B_{1} \setminus \{ x \in B_{1} \ | \ \hat{x}_1^{\top} \hat{\theta}_r-x^{\top} \hat{\theta}_r >\varepsilon_r \}$\; \label{line:select_S1_end}
			break\;
		}
		$r \leftarrow  r+1$\; \label{line:whileefirstepoch}
	}
	
	
	$\hat{x}^* \leftarrow$ output by \textsf{ALBA}$(S_1, \delta_1)$
	
	\Output{$\hat{x}^*$}
\end{algorithm}

\begin{algorithm}[t!]
	\caption{Computing a distribution $\lambda$ }\label{alg:compute_dist}
	\SetKwInOut{Input}{Input}
	\SetKwInOut{Output}{Output}
	\Input{$d$-base arms}
	
	Choose any $d$ super arms $\mathcal{X}_{\sigma}=\{b_1, \dots, b_d\}$ from $\mathcal{X}$, such that  $\mathrm{rank}(X)=d$ where $X=(b_1, \dots, b_d)$\;
	
	$\lambda^*_{\mathcal{X}_{\sigma}} \leftarrow \argmin_{\lambda \in \triangle( {\mathcal{X}_{\sigma})}} \max_{x \in {\mathcal{X}_{\sigma}}} x^{\top}M(\lambda)^{-1}x$ by 
	the entropy mirror descent algorithm of \cite{Tao2018} (see Algorithm~\ref{alg:descent} in Appendix~\ref{apx:G-opt}\OnlyInShort{ in the full version})
	\;
	
	
	$\alpha \leftarrow  \sqrt{md/\xi_{\min}(\widetilde{M}(\lambda^*_{\mathcal{X}_{\sigma}}))}$\;
	\Output{ $\lambda^*_{\mathcal{X}_{\sigma}}$ and $\alpha$}
\end{algorithm}

\vspace{1mm}
\noindent
{\bf Main algorithm. \ }
Now we present our proposed algorithm \textsf{PolyALBA} (see Algorithm~\ref{alg:PolyALBA} for its description), in which \textsf{ALBA} is invoked with $S=S_1$ with $|S_1|=d$.
Set $S_1$ is constructed by a novel preparation procedure in the first epoch ($q=0$).
In this preparation epoch, we first compute a fixed distribution $\lambda \in \triangle(\cX)$ that has a polynomial-size support and a key parameter $\alpha$ (line~\ref{line:computelambda}).
Then, based on $\lambda$ we apply static estimation to estimate $\theta$, until we see a big enough gap between the empirically best and $(d+1)$-th best actions (lines~\ref{line:whilebfirstepoch}--\ref{line:whileefirstepoch}).
The empirical top-$d$ actions, excluding those that also have big gaps to the  best one, form the set $S_1$ (lines~\ref{line:select_S1_start}--\ref{line:select_S1_end}), which is used to call \textsf{ALBA} to obtain the final result $\hat{x}^*$.

Note that, computing the empirical best $d+1$ super arms can be done in polynomial time by using \emph{Lawler's k-best procedure}~\cite{Lawler72}. 
This procedure only requires the existence of the efficient maximization oracle, which is satisfied in many combinatorial problems such as maximum matching, shortest paths and minimum spanning tree.
Moreover, the computational efficiency of \textsf{PolyALBA} is not merely owing to the Lawler's $k$-best procedure. 
In fact, even if previous BAI-LB algorithms 
apply the same procedure, they cannot run in polynomial time since they explicitly maintain  exponential-sized action set and sample on distributions with exponential supports. 
These render heavy computation and memory in every round of previous algorithms.
In contrast, we avoid the naive enumeration and sampling on the combinatorial space directly, and instead find empirical top-$d$ actions as representatives through a novel polynomial-time computation procedure.

%
%

\subsection{Theoretical Analysis}

Now we provide the sample complexity bound of  \textsf{PolyALBA}.
\begin{theorem}\label{thm:CPE-BL}
	With probability at least $1-\delta$, the \textsf{PolyALBA} algorithm (Algorithm \ref{alg:PolyALBA}) returns the best super arm $x^*$ with sample complexity 
	\begin{align*}
		O\Bigg( & \sum_{i=2}^{\left \lfloor  \frac{d}{2} \right \rfloor} \frac{c_0}{\Delta_i^2} (\ln \delta^{-1} + \ln |\mathcal{X}|+\ln\ln \Delta_i^{-1}) 
		\\
		& + \frac{c_0 d (\alpha \sqrt{m} + \alpha^2 )}{\Delta_{d+1}^2} \left(\ln \delta^{-1} + \ln |\mathcal{X}|+\ln\ln \Delta_{d+1}^{-1} \right) \Bigg),
	\end{align*}
	where $\alpha = \sqrt{md/\xi_{\min}(\widetilde{M}(\lambda^*_{\mathcal{X}_{\sigma}}))}$.
\end{theorem}

\vspace{1mm}
\noindent
{\bf Analysis of the statistical and computational efficiency. \ }
The first term in Theorem~\ref{thm:CPE-BL} is for the remaining epochs required by subroutine \textsf{ALBA} and
	the second term is for the preparation procedure.
As shown in Theorem~\ref{thm:CPE-BL}, our sample complexity bound has lighter dependence on $1/\Delta^2_{\min}$, compared with the existing result (see Table~\ref{table:results}).
Now we explain the key role for the polynomial-time complexity of \textsf{PolyALBA} in the first epoch played by the distribution $\lambda^*_{\mathcal{X}_{\sigma}}$ and parameter $\alpha$.
Notice that even if we employ a uniform distribution on a polynomial-size support $\cX_{\sigma} \subseteq \cX$, i.e., $\lambda_{\mathcal{X}_{\sigma}}=(1/|\cX_{\sigma}|)_{x \in \cX_{\sigma}}$,
computing the maximal confidence bound $\max_{x \in \cX}\|x\|_{M(\lambda_{\mathcal{X}_{\sigma}})^{-1}}$ is NP-hard,
while many (UCB-based) algorithms in LB ignore this issue and simply use a brute force method.
In contrast, \textsf{PolyALBA} utilizes G-optimal design~\cite{pukelsheim2006} and 
	runs in polynomial time while guaranteeing the optimality. 
In the following lemma, we show that $\alpha\sqrt{d}$ gives the upper bound on the maximal ellipsoidal norm associated to ${M(\lambda^*_{\mathcal{X}_{\sigma}})^{-1}}$.
\begin{lemma} \label{lemma: alpha}
	For $\lambda^*_{\mathcal{X}_{\sigma}}$ and $\alpha$ obtained by Algorithm~\ref{alg:compute_dist}, it holds that
	$\max_{x \in \mathcal{X}} \|x\|_{M(\lambda^*_{\mathcal{X}_{\sigma}})^{-1}} \leq  \alpha \sqrt{d},$
	where $\alpha = \sqrt{md/\xi_{\min}(\widetilde{M}(\lambda^*_{\mathcal{X}_{\sigma}}))}$.   
\end{lemma}
From the equivalence theorem for optimal experimental designs (Proposition~\ref{proposition:equivalence} in Appendix~\ref{apx:equivalence_theorem}\OnlyInShort{ in the full version}), it holds that 
$\min_{\lambda \in \triangle(\cX)} \max_{x \in \cX} \| x \|_{M(\lambda)^{-1}}=\sqrt{d}$.
From this fact and Lemma~\ref{lemma: alpha}, we see that $\lambda^*_{\mathcal{X}_{\sigma}}$ is $\alpha \ (\geq 1)$-approximate solution to $\min_{\lambda \in \triangle(\cX)} \max_{x \in \cX} \| x \|_{M(\lambda)^{-1}}$ where $\cX$ can be defined by general combinatorial constraints.
Note that $\alpha$ can be easily obtained by computing $\xi_{\min}(\widetilde{M}(\lambda^*_{\mathcal{X}_{\sigma}}))$ (recall that ${\widetilde{M}({\lambda})}=\sum_{x \in \mathrm{supp}(\lambda)}x x^{\top}$).
Therefore, by employing $\lambda^*_{\mathcal{X}_{\sigma}}$ and a prior knowledge of its approximation ratio $\alpha$,
we can guarantee that the preparation sampling scheme identifies a set $S_1$ containing the optimal super arm $x^*$ with high probability.
In the remaining epochs,
\textsf{PolyALBA} can successfully focus on sampling near-optimal super arms by \textsf{ALBA} owing to the optimality of $S_1$.
Note that  $\alpha=1$ if we compute $\min_{\lambda \in \triangle(\cX)} \max_{x \in \cX} \| x \|_{M(\lambda)^{-1}}$ exactly.
If we approximately solve it,
$\alpha$ is independent on the arm-selection ratio but it can depend on the support of $\lambda$. 
For further discussion on improving $\alpha$, please see Appendix~\ref{apx:G-opt}\OnlyInShort{ in the full version}.

\vspace{1mm}
\noindent
{\bf Discussion on the optimality. \ }
Fiez et al.~\shortcite{Fiez2019} give a sample complexity lower bound for BAI-LB (see Table~\ref{table:results}) and propose a nearly (within a logarithmic factor) optimal algorithm \textsf{RAGE} with sample complexity of $O \left(\sum^{\left \lfloor \log_2(4/\Delta_{\min})  \right \rfloor}_{t=1} 2(2^t)^2\tilde{\rho}(\mathcal{Y}(S_t)) \log(t^2|\mathcal{X}|^2/\delta) \right)$.
Note that the existing lower bound~\cite{Fiez2019} and nearly (or asymptotically) optimal algorithms~\cite{Fiez2019,Degenne+2020,Katz+2020} do not consider computational efficiency for combinatorially-large $|\cX|$, and the lower bound for polynomial-time CPE-BL algorithms is still an open problem. 

When compared to the lower bound~\cite{Fiez2019}, there exists a family of instances such that $\Delta_{\left \lfloor d/ 2^{(t-2)} \right \rfloor+1}=4\cdot2^{-t}, \  t=2, 3, \dots, \log_2(\frac{4}{\Delta_{\min}})$, in which our \textsf{PolyALBA} achieves $O (\sum^{\left \lfloor \log_2(4/\Delta_{\min})  \right \rfloor}_{t=2} 2(2^t)^2\tilde{\rho}(\mathcal{Y}(S_t)) \log(t^2|\mathcal{X}|^2/\delta)+ m d  \xi_{\max}({\tilde{M}^{-1}({\lambda})}) \tilde{\rho}(\mathcal{Y}(S_1)) \log(|\mathcal{X}|^2/\delta) )$ sample complexity (see Appendix~\ref{apx:polyalba_optimality}\OnlyInShort{ in the full version} for more details). 
When ignoring a logarithmic factor and with sufficiently small $\Delta_{\min}$, the additional term related to $\xi_{\max}({\tilde{M}^{-1}({\lambda})})$ is absorbed and the result matches the lower bound, 
 which shows superiority over other heavily $\Delta_{\min}$-dependent algorithms~\cite{Soare2014,Karin16,Kuroki+19}. 
Note that the term related to $\xi_{\max}({\tilde{M}^{-1}({\lambda})})$ can be viewed as the cost for achieving computational efficiency. 


To our best knowledge, our \textsf{PolyALBA} is the first polynomial-time adaptive algorithm that works for CPE-BL with general combinatorial structures and achieves nearly optimal sample complexity for a family of problem instances.

\section{Combinatorial Pure Exploration with Partial Linear Feedback (CPE-PL)}\label{sec:CPE-PL}

In this section, we present the first polynomial-time algorithm {\sf GCB-PE} for CPE-PL with sample complexity analysis, and
discuss its further improvements via a non-uniform allocation strategy. We also give  practical applications for CPE-PL and explain the corresponding global observer set and sample complexity result in these scenarios.

\subsection{Algorithm Procedure} \label{sec:GCBPEprocedure}
We illustrate {\sf GCB-PE} in Algorithm \ref{alg:GCB-PE}.
{\sf GCB-PE} estimates the environment vector $\theta$ by repeatedly pulling the global observer set $\sigma=\{x_1, x_2, \dots, x_{|\sigma|}\}$, which in turn helps estimate the expected rewards $\bar{r}(x, \theta)$ of all super arms $x \in \cX$ using the Lipschitz continuity (Assumption \ref{assumption_Lipschitz}). We call one pull of global observer set $\sigma$ \emph{one exploration round}, the specific procedure of which is described as follows: for the $n$-th exploration round, the learner plays all actions in $\sigma=\{x_1, x_2, \dots, x_{|\sigma|}\}$ once and respectively observes feedback $y_1, y_2, \dots, y_{|\sigma|}$, the stacked vector of which is denoted by $\vec{y}_n=(y_1; y_2; \dots; y_{|\sigma|})$. The estimate of environment vector $\theta$ in this exploration round is $\hat{\theta}_n= M_{\sigma}^+ \vec{y}_n$, where $M_{\sigma}^+$ is the Moore-Penrose pseudoinverse of $M_{\sigma}$. From Assumption \ref{assumption_global}, we have $\mathbb{E}[\hat{\theta}_n]=\theta$. Then, we can use the independent estimates in multiple rounds, i.e.,  $\hat{\theta}(n)= \frac{1}{n} \sum_{j=1}^{n} \hat{\theta}_j$, to obtain an accurate estimate of $\theta$.

Similar to Lin et al.~\shortcite{LinTian2014}, we define a constant $\beta_\sigma:=\max_{\eta_1, \cdots, \eta_{|\sigma|}  \in [-1,1]^{d}}  \|  (M_{\sigma}^{\top} M_{\sigma})^{-1} \sum_{i=1}^{|\sigma|} M_{x_i}^{\top} M_{x_i} \eta_i  \|_2$, which only depends on global observer set $\sigma$, and bounds the estimate error of one exploration round, i.e., for any $n$,
$\|\hat{\theta}_n -\theta \|_2 \leq \beta_\sigma$, the proof of which is given in Appendix~\ref{apx:eq_beta_sigma}\OnlyInShort{ in the full version}.
Based on $\beta_\sigma$, we further design a global confidence radius   $\textup{rad}_n = \sqrt{ 2 \beta_\sigma^2 \log( 4n^2e^2/ \delta ) /n }$ for the estimate $\hat{\theta}(n)$, and show that with high probability, $\textup{rad}_n$ bounds the estimate error of $\hat{\theta}(n)$.

Compared with {\sf GCB} in Lin et al.~\shortcite{LinTian2014}, which works for the regret minimization metric of the combinatorial partial monitoring game with linear feedback problem,  {\sf GCB-PE} targets the best action identification and mainly controls the stopping time of the exploration phase rather than balancing the frequency of  exploration and exploitation phases. 
For the pure exploration metric, our global confidence radius $\textup{rad}_n$ is novelly designed to bound the estimate error. In addition, the stopping condition, which uses the designed confidence radius and Lipschitz continuity of the expected reward function, is also novelly adopted to fit the CPE-PL setting.

The computational efficiency of {\sf GCB-PE} relies on the polynomial-time offline maximization oracle for the specific combinatorial instance, which is used in the two $\mathtt{\argmax}$ operations in {\sf GCB-PE}. It is reasonable to assume the existence of polynomial-time offline maximization oracle, otherwise we cannot efficiently address the exponentially large action space even if the real environment vector $\theta$ is known.


	
	

\begin{algorithm}[t]
	\caption{\textsf{GCB-PE}}\label{alg:GCB-PE}
	\SetKwInOut{Input}{Input}
	\SetKwInOut{Output}{Output}
	\Input{Confidence level $\delta$, global observer set $\sigma$, constant $\beta_\sigma$, Lipschitz constant $L_p$}
	\For{$s =1,\dots, |\sigma|$}
	{
		Pull $x_s$ in observer set $\sigma$, and observe $y_s$\;
	}
	$n \leftarrow 1$\;
	$\vec{y}_1 \leftarrow (y_1; y_2; \dots; y_{|\sigma|})$\;
	$\hat{\theta}_1 \leftarrow  M_{\sigma}^+ \vec{y}_1$ and $\hat{\theta}(1) \leftarrow \hat{\theta}_1 $\;
	
	
	\While{true}
	{
		$\hat{x} \leftarrow \argmax_{x \in \cX} \bar{r}(x, \hat{\theta}(n))$\;
		$\hat{x}^- \leftarrow  \argmax_{x \in \cX \setminus\{ \hat{x} \}} \bar{r}(x, \hat{\theta}(n))$\;
		$\textup{rad}_n \leftarrow \sqrt{ \frac{2 \beta_\sigma^2 \log( \frac{4n^2e^2}{\delta} ) }{n} }$\;
		\If{ $\bar{r}(\hat{x}, \hat{\theta}(n)) - \bar{r}(\hat{x}^-, \hat{\theta}(n)) > 2L_p \cdot \textup{rad}_n$ } 
		{
			\Return {$\hat{x}$}\;
		}
		\Else
		{
			\For{$s =1,\dots, |\sigma|$}
			{
				Pull $x_s$ in observer set $\sigma$, and observe $y_s$\;
			}
			$n \leftarrow n+1$\;
			$\vec{y}_n \leftarrow (y_1; y_2; \dots; y_{|\sigma|})$\;
			$\hat{\theta}_n \leftarrow  M_{\sigma}^+ \vec{y}_n$\;
			$\hat{\theta}(n) \leftarrow  \frac{1}{n} \sum_{j=1}^{n} \hat{\theta}_j$\;
		}
	}
	\Output{ $\hat{x}$}
\end{algorithm}

\subsection{Theoretical Analysis}
We give the sample complexity of {\sf GCB-PE} below. 

\begin{theorem} \label{thm:GCB_ub}
	With probability at least $1-\delta$, the \textsf{GCB-PE} algorithm (Algorithm \ref{alg:GCB-PE}) will return the optimal super arm $x^*$ with sample complexity 
	$$
	O \left( \frac{|\sigma| \beta_\sigma^2 L_p^2}{\Delta_{\textup{min}}^2} \log \left( \frac{ \beta_\sigma^2 L_p^2}{\Delta_{\textup{min}}^2 \delta } \right )   \right ) ,
	$$
	where $|\sigma| \leq d$.
\end{theorem}


When the expected reward function is linear, i.e. $\bar{r}(x,\theta)=x^{\top} \theta$, we have $L_p=\sqrt{m}$, where $m$ ($\leq d$) is the maximum number of base arms a super arm contains.
In addition, $\beta_\sigma=\textup{Poly}(d)$ in several practical applications of CPE-PL (see Section~\ref{sec:PL_application} for our detailed discussion).


\vspace{1mm}
\noindent
{\bf Discussion on the optimality. \ }
While the sample complexity  of $\textsf{GCB-PE}$ is sometimes worse than the CPE-BL or BAI-LB algorithms ($\textsf{PolyALBA}$, $\textsf{ALBA}$ and $\textsf{RAGE}$), it solves a more general class of problems than CPE-BL and BAI-LB. 
We emphasize that our contribution mainly focuses on proposing the first polynomial-time algorithm \textsf{GCB-PE} that simultaneously addresses combinatorial action apace, partial linear feedback and nonlinear reward function.
%
%
\vspace{1mm}
\noindent
{\bf On non-uniform or adaptive allocation strategy. \ }
\textsf{GCB-PE} can be further improved by employing a \emph{non-uniform} allocation strategy when considering the global observer set $\sigma$ with multiplicity: we can obtain such an allocation by solving an optimization $\argmin_{\lambda \in \triangle(\sigma)} \beta_{\sigma}(\lambda)$ and rounding the result,
where $ \beta_{\sigma}(\lambda):=\max_{\eta_1, \cdots, \eta_{|\sigma|}  \in [-1,1]^{d}}  \|  (M_{\sigma}^{\top} M_{\sigma})^{-1} \sum_{i=1}^{|\sigma|}\lambda_i M_{x_i}^{\top} M_{x_i} \eta_i  \|_2$.
Since uniform sampling is not essential in our analysis, the proposed improvement for \textsf{GCB-PE} via non-uniform allocation does not violate Assumption~\ref{assumption_global} and keeps our theoretical analysis.
\textsf{GCB-PE} is a static algorithm, and we leave the study of adaptive strategies for CPE-PL as future work. 
In Appendix~\ref{apx:clucb}\OnlyInShort{ in the full version}, we discuss a fully-adaptive algorithm for CPE-BL (special case of CPE-PL), and show that the result depends on a non-controllable term $M(\lambda)^{-1}$, which indicates that the static control may be required to deal with linear feedback efficiently.

\subsection{Applications for GCB-PL} \label{sec:PL_application}

CPE-PL characterizes more flexible feedback structures than CPE-BL (or BAI-LB) and finds many real-world applications. Below we present two practical applications and discuss the global observer set (Assumption~\ref{assumption_global}) and parameter $\beta_\sigma$.  

\vspace{1mm}
\noindent
{\bf Online ranking. \ }
Consider that a company wishes to recommend their products to users by presenting the ranked list of items. Due to user burden constraints and privacy concerns, collecting a large amount of data on the relevance of all items might be infeasible, and thus the company usually collects the relevance of only the top-ranked item~\cite{Sougara-Ambuj2015,Sougata-Ambuj2016, Sougata-Ambuj2017}.
In this scenario, a learner selects a permutation of $d$ items (each action $x$ is a permutation) at each step, and observes the relevance of the top-ranked item, i.e., $M_x$ contains a single row with $1$ in the place of the top-ranked item and $0$ everywhere else. The objective is to identify the best permutation as soon as possible. 
Then, we can construct a global observable set $\sigma$ to be the set of any $d$ actions which places a distinct item at top. Here $M_{\sigma}$ is the $d \times d$ identity matrix and $\beta_{\sigma}=\sqrt{d}$.
	
\vspace{1mm}
\noindent
{\bf Task assignments in crowdsourcing. \ }
Consider that an employer wishes to assign crowdworkers to tasks with high quality performance, and it wants to avoid
	the high cost and the privacy concern of collecting each individual worker-task pair performance~\cite{LinTian2014}.
Thus, the employer sequentially chooses an assignment from $N$ workers to $M$ tasks (each action $x$ is a worker-task matching) and only collects the sum of performance feedback for $1 \leq s<N$ matched worker-task pairs, i.e., $M_x$ contains a single row with $1$s in the places of $s$ matched pairs and $0$ everywhere else. The objective is to find the best worker-task matching as soon as possible.
For $1\le s<N$, Lin et al.~\shortcite{LinTian2014} provide a systematic method to construct a global observer set.

\section{Experiments}
We conduct experiments for CPE-BL and CPE-PL on the matching and top-$k$ instances, and compare our algorithms with the state-of-the-arts in both running time and sample complexity.
Due to the space limit, here we only present the results on matchings and defer the top-$k$ results with discussion on $\Delta_{\textup{min}}$-dependence to Appendix~\ref{apx:top_k_experiments}\OnlyInShort{ in the full version}.

We evaluate all the compared algorithms on Intel Xeon E5-2640 v3 CPU at 2.60GHz with 132GB RAM. 
For both CPE-BL and CPE-PL, we set action space $\cX$ as matchings in $3$-by-$3$,  $4$-by-$4$ and $5$-by-$5$ complete bipartite graphs. 
The dimension $d$, i.e. the number of edges, is set from $9$ to $25$. The number of matchings $|\cX|$ are set from $12$ to $480$. $\theta_1, \dots, \theta_d$ is set as a geometric sequence in $[0,1]$. 
We simulate the random feedback for action $x$ by a Gaussian distribution with mean of $x^\top \theta$ and unit variance. 
For CPE-PL, 
we use the full-bandit feedback as CPE-BL ($M_x=x^\top$) but a nonlinear reward function $\bar{r}(x, \theta)=x^\top \theta / \|x\|_1$.
For each algorithm, we perform $20$ independent runs and present the average running time and sample complexity with $95\%$ confidence intervals across runs. 
In the experiments, $\textsf{RAGE}$~\cite{Fiez2019} reports memory errors when $|\cX|>48$ due to its heavy memory burden, and thus we only obtain its results on small-$|\cX|$ instances. 
For $\textsf{PolyALBA}$, $\textsf{ALBA}$ and $\textsf{RAGE}$, we obtain the same sample complexity in different runs, since these algorithms compute the required samples at the beginning of each phase and then perform the fixed samples.

\begin{figure}[t]
	\centering
	\subfigure[CPE-BL]{\label{fig:matching_BL}
		\includegraphics[width=0.40\columnwidth]{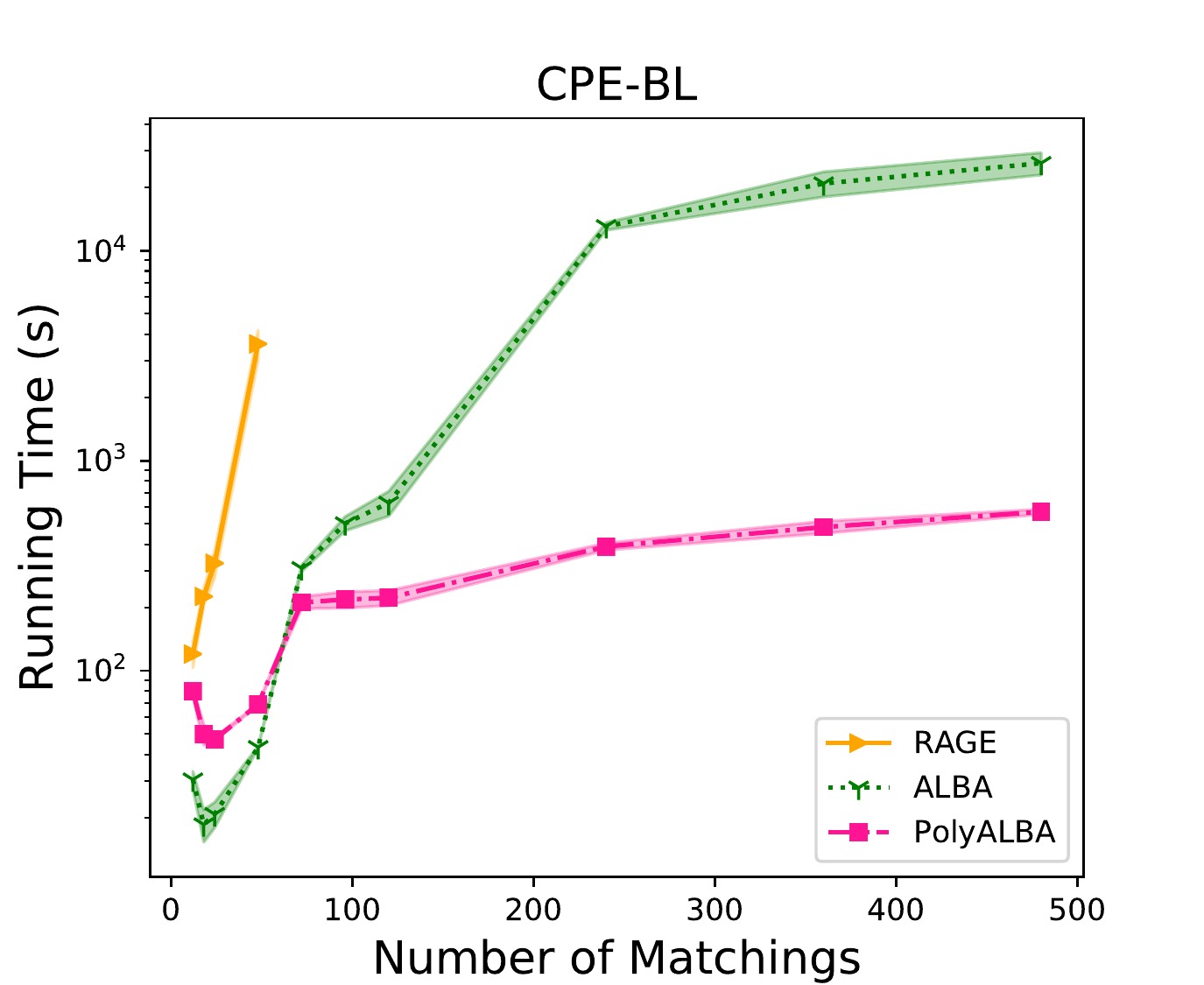}
		\includegraphics[width=0.40\columnwidth]{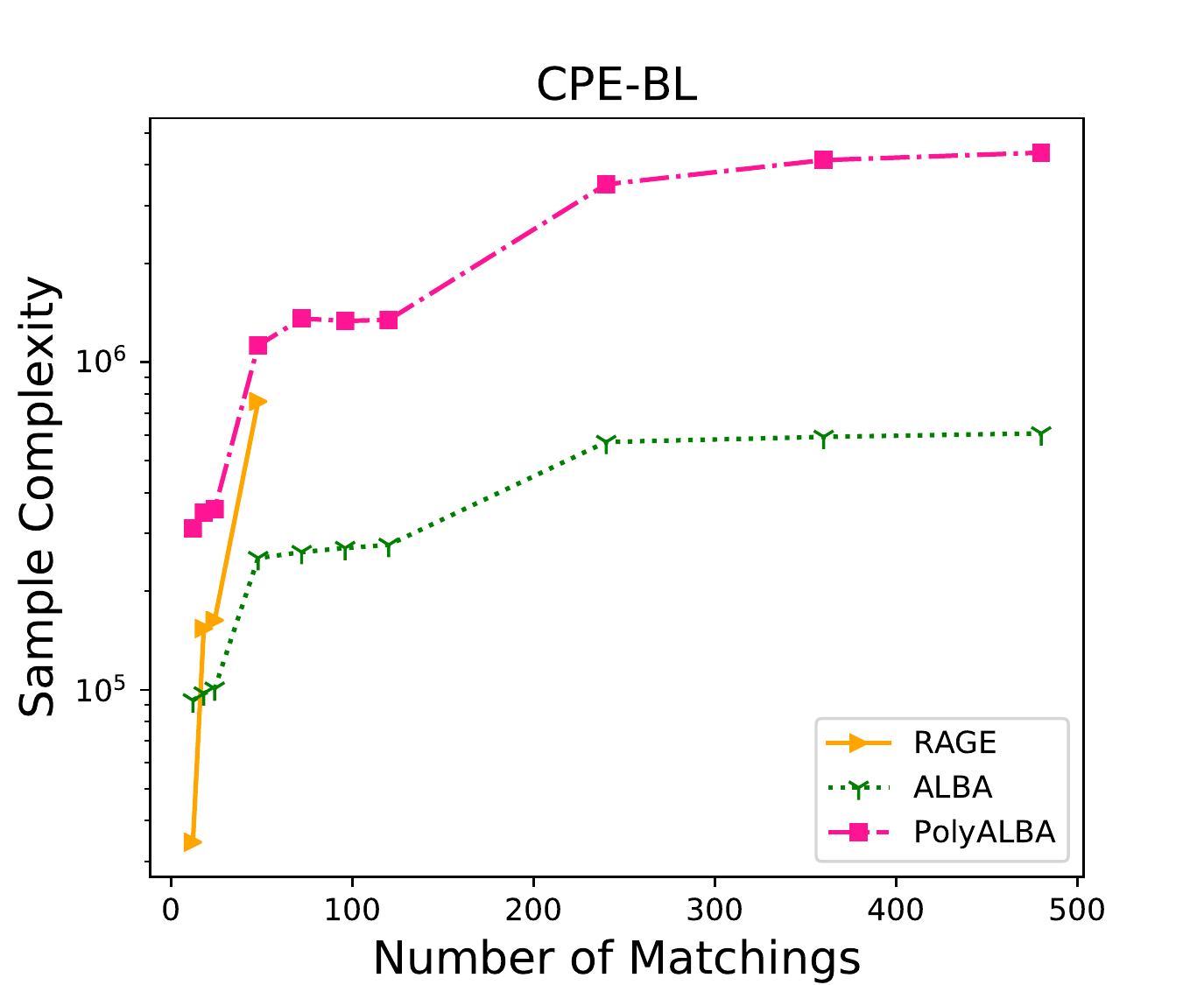}
	}
	\quad
	\subfigure[CPE-PL]{\label{fig:matching_PL}
		\includegraphics[width=0.40\columnwidth]{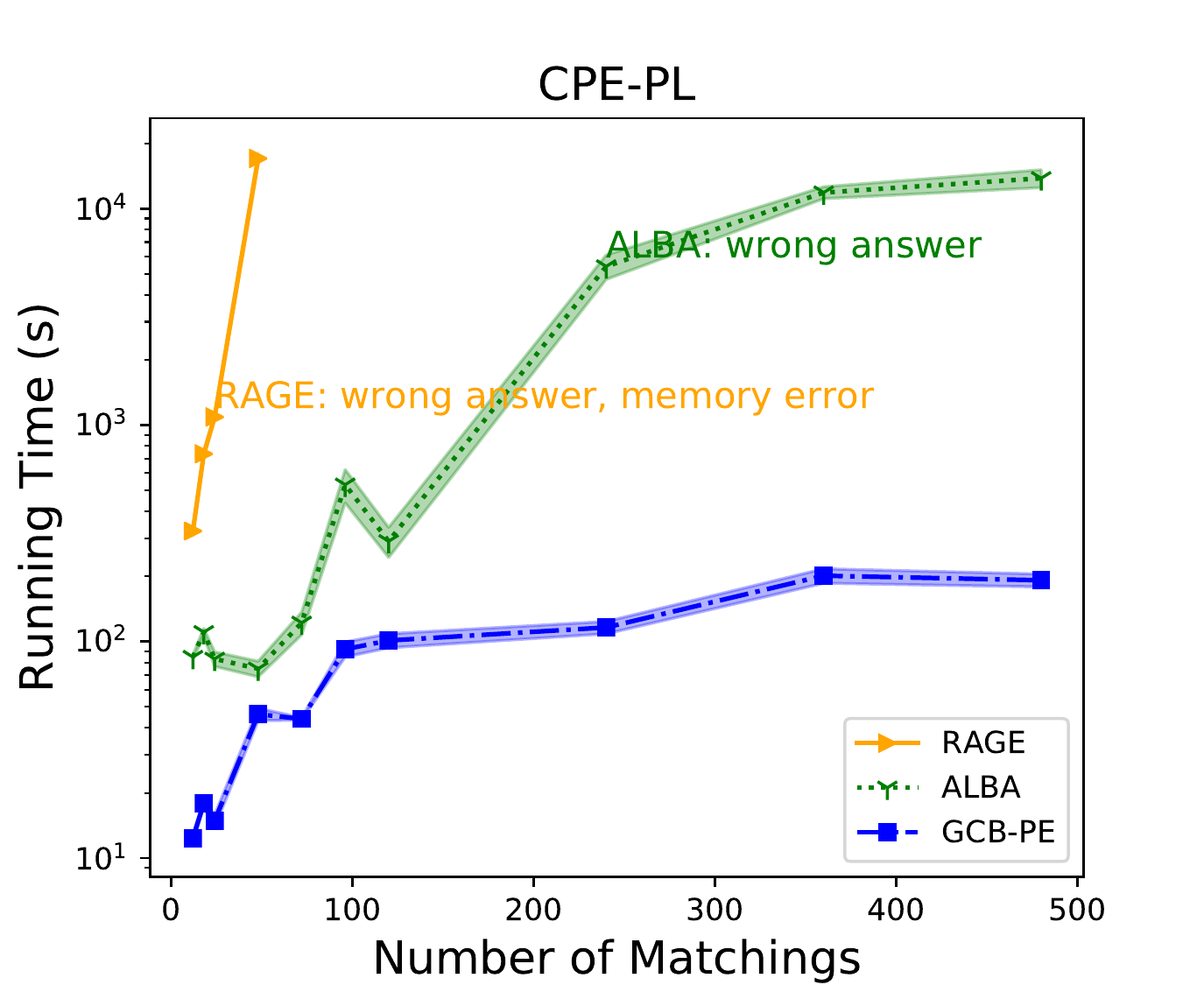}
		\includegraphics[width=0.40\columnwidth]{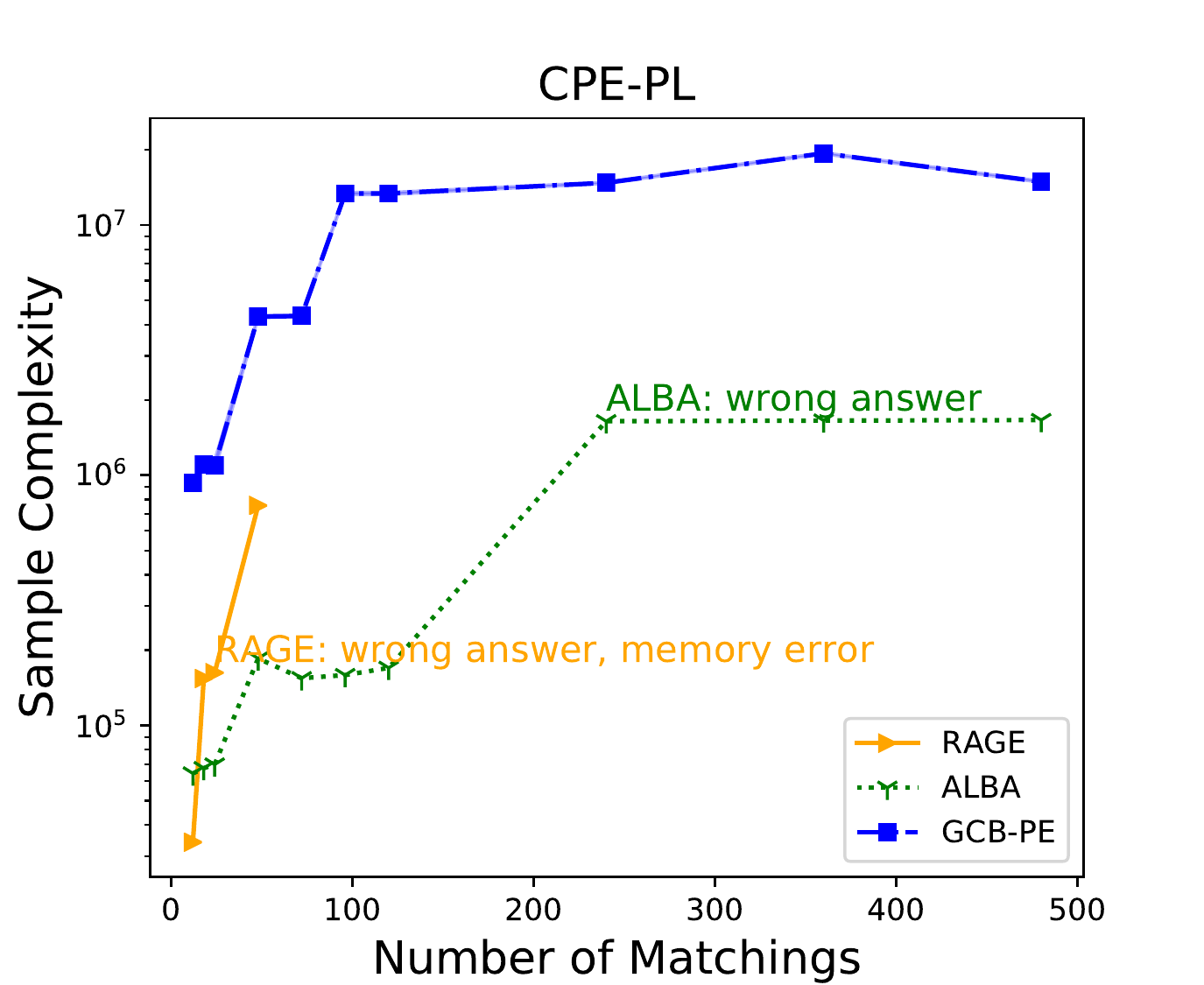}
	}
	\caption{Experimental results of running time and sample complexity for CPE-BL and CPE-PL.}
	\label{fig:matching_experiments}
\end{figure}

\vspace{1mm}
\noindent
{\bf Experiments for CPE-BL. \ }
For CPE-BL, we compare our $\textsf{PolyALBA}$ with the state-of-the-art BAI-LB algorithms $\textsf{ALBA}$ and $\textsf{RAGE}$ in running time and sample complexity. 
As shown in Figure~\ref{fig:matching_BL} with a logarithmic y-axis, our $\textsf{PolyALBA}$ runs about two orders of magnitude faster than $\textsf{ALBA}$ and $\textsf{RAGE}$, and the running time of $\textsf{PolyALBA}$ increases more slowly than the others as $|\cX|$ increases.
Due to the extra preparation epoch, $\textsf{PolyALBA}$ has a higher sample complexity, but we argue that in practice one has to keep the computation time low first to make an algorithm useful, and for that matter 
	$\textsf{ALBA}$ and $\textsf{RAGE}$ are too slow to run and $\textsf{PolyALBA}$ is the only feasible option. 


\vspace{1mm}
\noindent
{\bf Experiments for CPE-PL. \ }
For CPE-PL, we compare $\textsf{GCB-PE}$ with BAI-LB algorithms $\textsf{ALBA}$ and $\textsf{RAGE}$ in running time and sample complexity on a more challenging nonlinear reward task. 
In the experiments for CPE-PL, $\textsf{ALBA}$ and $\textsf{RAGE}$ return wrong answers because they are not designed to handle nonlinear reward functions.
Nevertheless, we can still analyze the running times presented in Figure~\ref{fig:matching_PL}.
It shows that our $\textsf{GCB-PE}$ runs two orders of magnitude faster than $\textsf{ALBA}$ and $\textsf{RAGE}$ while reporting the correct answer. 
In addition, as $|\cX|$ increases, the running time of $\textsf{GCB-PE}$ increases in a much slower pace than the others.
The experimental results demonstrate the capability of $\textsf{GCB-PE}$ to simultaneously deal with combinatorial action space, nonlinear reward function and partial feedback in a computationally efficient way.

\section{ Future Work}\label{sec:conclusion}
%
%
%
%
%

There are several interesting directions worth further investigation.
First, it is open to prove a lower bound of polynomial-time algorithms for both CPE-PL and CPE-BL.
Another challenging direction is to design efficient algorithms for specific combinatorial cases to choose the global observer set $\sigma$ and the distribution $\lambda^*_{\mathcal{X}_{\sigma}}$, and derive specific sample complexity bounds.
Furthermore, the extension of CPE-PL to nonlinear feedback is also a practical and valuable problem.


\section*{Acknowledgement}
YK would like to thank Masashi Sugiyama and Junya Honda for helpful comments on the manuscript, and Kento Nozawa for his support to use the server.
YK is supported by Microsoft Research Asia, KAKENHI 18J23034, and JST ACT-X 1124477.

\bibliography{mybib}

\begin{thebibliography}{39}
\providecommand{\natexlab}[1]{#1}
\providecommand{\url}[1]{\texttt{#1}}
\providecommand{\urlprefix}{URL }
\expandafter\ifx\csname urlstyle\endcsname\relax
  \providecommand{\doi}[1]{doi:\discretionary{}{}{}#1}\else
  \providecommand{\doi}{doi:\discretionary{}{}{}\begingroup
  \urlstyle{rm}\Url}\fi

\bibitem[{Abbasi-{Y}adkori, P\'{a}l, and Szepesv\'{a}ri(2011)}]{yadkori11}
Abbasi-{Y}adkori, Y.; P\'{a}l, D.; and Szepesv\'{a}ri, C. 2011.
\newblock Improved Algorithms for Linear Stochastic Bandits.
\newblock In \emph{Proc. NIPS'14}, 2312--2320.

\bibitem[{Audibert, Bubeck, and Munos(2010)}]{Audibert2010}
Audibert, J.-Y.; Bubeck, S.; and Munos, R. 2010.
\newblock Best Arm Identification in Multi-Armed Bandits.
\newblock In \emph{Proc. COLT'10}, 41--53.

\bibitem[{Bubeck, Wang, and Viswanathan(2013)}]{bubeck2013}
Bubeck, S.; Wang, T.; and Viswanathan, N. 2013.
\newblock Multiple identifications in multi-armed bandits.
\newblock In \emph{Proc. ICML'13}, 258--265.

\bibitem[{Chaudhuri and Tewari(2015)}]{Sougara-Ambuj2015}
Chaudhuri, S.; and Tewari, A. 2015.
\newblock Online ranking with top-1 feedback.
\newblock In \emph{Proc.~AISTATS'15}, 129--137.

\bibitem[{Chaudhuri and Tewari(2016)}]{Sougata-Ambuj2016}
Chaudhuri, S.; and Tewari, A. 2016.
\newblock Phased Exploration with Greedy Exploitation in Stochastic
  Combinatorial Partial Monitoring Games.
\newblock In \emph{Proc. NIPS'16}, 2433--2441.

\bibitem[{Chaudhuri and Tewari(2017)}]{Sougata-Ambuj2017}
Chaudhuri, S.; and Tewari, A. 2017.
\newblock Online learning to rank with top-k feedback.
\newblock \emph{The Journal of Machine Learning Research} 18(1): 3599--3648.

\bibitem[{Chen and Li(2015)}]{Chen2015}
Chen, L.; and Li, J. 2015.
\newblock On the Optimal Sample Complexity for Best Arm Identification.
\newblock \emph{arXiv preprint} arXiv:1511.03774.

\bibitem[{Chen et~al.(2014)Chen, Lin, King, Lyu, and Chen}]{Chen2014}
Chen, S.; Lin, T.; King, I.; Lyu, M.~R.; and Chen, W. 2014.
\newblock Combinatorial Pure Exploration of Multi-Armed Bandits.
\newblock In \emph{Proc. NIPS'14}, 379--387.

\bibitem[{Chen et~al.(2020)Chen, Du, Huang, and Zhao}]{chen2020cpe_db}
Chen, W.; Du, Y.; Huang, L.; and Zhao, H. 2020.
\newblock Combinatorial Pure Exploration for Dueling Bandit.
\newblock In \emph{Proc. ICML'20}, 1531--1541.

\bibitem[{Dani, Hayes, and Kakade(2008)}]{Dani2008}
Dani, V.; Hayes, T.~P.; and Kakade, S.~M. 2008.
\newblock Stochastic Linear Optimization Under Bandit Feedback.
\newblock In \emph{Proc. COLT'08}, 355--366.

\bibitem[{Degenne et~al.(2020)Degenne, Menard, Shang, and Valko}]{Degenne+2020}
Degenne, R.; Menard, P.; Shang, X.; and Valko, M. 2020.
\newblock Gamification of Pure Exploration for Linear Bandits.
\newblock In \emph{Proc. ICML'20}, 2432--2442.

\bibitem[{Even-Dar, Mannor, and Mansour(2006)}]{Even2006}
Even-Dar, E.; Mannor, S.; and Mansour, Y. 2006.
\newblock Action elimination and stopping conditions for the multi-armed bandit
  and reinforcement learning problems.
\newblock \emph{Journal of Machine Learning Research} 7: 1079--1105.

\bibitem[{Fiez et~al.(2019)Fiez, Jain, Jamieson, and Ratliff}]{Fiez2019}
Fiez, T.; Jain, L.; Jamieson, K.~G.; and Ratliff, L. 2019.
\newblock Sequential Experimental Design for Transductive Linear Bandits.
\newblock In \emph{Proc. NeurIPS'~19}, 10667--10677.

\bibitem[{Gabillon, Ghavamzadeh, and Lazaric(2012)}]{Gabillon2012}
Gabillon, V.; Ghavamzadeh, M.; and Lazaric, A. 2012.
\newblock Best Arm Identification: A Unified Approach to Fixed Budget and Fixed
  Confidence.
\newblock In \emph{Proc. NIPS'12}, 3212--3220.

\bibitem[{Gabillon et~al.(2011)Gabillon, Ghavamzadeh, Lazaric, and
  Bubeck}]{Gabillon+2011_multibandit}
Gabillon, V.; Ghavamzadeh, M.; Lazaric, A.; and Bubeck, S. 2011.
\newblock Multi-Bandit Best Arm Identification.
\newblock In \emph{Proc. NIPS'11}, 2222--2230.

\bibitem[{Grötschel, Lovász, and Schrijver(1981)}]{Grotschel1981}
Grötschel, M.; Lovász, L.; and Schrijver, A. 1981.
\newblock The ellipsoid method and its consequences in combinatorial
  optimization.
\newblock \emph{Combinatorica} 1: 169--197.

\bibitem[{Huang, Liu, and Ding(2008)}]{Huang2008}
Huang, S.; Liu, X.; and Ding, Z. 2008.
\newblock Opportunistic Spectrum Access in Cognitive Radio Networks.
\newblock In \emph{Proc. INFOCOM'08}, 1427--1435.

\bibitem[{Huang et~al.(2018)Huang, Ok, Li, and Chen}]{huang2018}
Huang, W.; Ok, J.; Li, L.; and Chen, W. 2018.
\newblock Combinatorial Pure Exploration with Continuous and Separable Reward
  Functions and Its Applications.
\newblock In \emph{Proc. IJCAI~'18}, 2291--2297.

\bibitem[{Jedra and Proutiere(2020)}]{jedra2020optimal}
Jedra, Y.; and Proutiere, A. 2020.
\newblock Optimal Best-arm Identification in Linear Bandits.
\newblock \emph{to appear in NeurIPS'20} .

\bibitem[{Kalyanakrishnan and Stone(2010)}]{kalyanakrishnan2010}
Kalyanakrishnan, S.; and Stone, P. 2010.
\newblock Efficient Selection of Multiple Bandit Arms: Theory and Practice.
\newblock In \emph{Proc. ICML'10}, 511--518.

\bibitem[{Kalyanakrishnan et~al.(2012)Kalyanakrishnan, Tewari, Auer, and
  Stone}]{kalyanakrishnan2012}
Kalyanakrishnan, S.; Tewari, A.; Auer, P.; and Stone, P. 2012.
\newblock {PAC} Subset Selection in Stochastic Multi-armed Bandits.
\newblock In \emph{Proc. ICML'12}, 655--662.

\bibitem[{Karnin(2016)}]{Karin16}
Karnin, Z.~S. 2016.
\newblock Verification Based Solution for Structured MAB Problems.
\newblock In \emph{Proc. NIPS'~16}, 145--153.

\bibitem[{Katz-Samuels et~al.(2020)Katz-Samuels, Jain, Karnin, and
  Jamieson}]{Katz+2020}
Katz-Samuels, J.; Jain, L.; Karnin, Z.; and Jamieson, K. 2020.
\newblock An Empirical Process Approach to the Union Bound: Practical
  Algorithms for Combinatorial and Linear Bandits.
\newblock \emph{arXiv preprint arXiv:2006.11685} .

\bibitem[{Kaufmann, Capp{\'e}, and Garivier(2016)}]{kaufmann2016}
Kaufmann, E.; Capp{\'e}, O.; and Garivier, A. 2016.
\newblock On the complexity of best-arm identification in multi-armed bandit
  models.
\newblock \emph{Journal of Machine Learning Research} 17: 1--42.

\bibitem[{Kawase and Sumita(2019)}]{Kawase2019}
Kawase, Y.; and Sumita, H. 2019.
\newblock Randomized strategies for robust combinatorial optimization.
\newblock In \emph{Proc. AAAI~'19}, 7876--7883.

\bibitem[{Kiefer and Wolfowitz(1960)}]{kiefer_wolfowitz_1960}
Kiefer, J.; and Wolfowitz, J. 1960.
\newblock The Equivalence of Two Extremum Problems.
\newblock \emph{Canadian Journal of Mathematics} 12: 363–366.

\bibitem[{Kuroki et~al.(2020{\natexlab{a}})Kuroki, Miyauchi, Honda, and
  Sugiyama}]{Kuroki2020}
Kuroki, Y.; Miyauchi, A.; Honda, J.; and Sugiyama, M. 2020{\natexlab{a}}.
\newblock Online Dense Subgraph Discovery via Blurred-Graph Feedback.
\newblock In \emph{Proc. ICML'20}, 5522--5532.

\bibitem[{Kuroki et~al.(2020{\natexlab{b}})Kuroki, Xu, Miyauchi, Honda, and
  Sugiyama}]{Kuroki+19}
Kuroki, Y.; Xu, L.; Miyauchi, A.; Honda, J.; and Sugiyama, M.
  2020{\natexlab{b}}.
\newblock Polynomial-Time Algorithms for Multiple-Arm Identification with
  Full-Bandit Feedback.
\newblock \emph{Neural Computation} 32(9): 1733--1773.

\bibitem[{Lawler(1972)}]{Lawler72}
Lawler, E.~L. 1972.
\newblock A Procedure for Computing the K Best Solutions to Discrete
  Optimization Problems and Its Application to the Shortest Path Problem.
\newblock \emph{Management Science} 18(7): 401--–405.

\bibitem[{Lin et~al.(2014)Lin, Abrahao, Kleinberg, Lui, and Chen}]{LinTian2014}
Lin, T.; Abrahao, B.; Kleinberg, R.; Lui, J.; and Chen, W. 2014.
\newblock Combinatorial partial monitoring game with linear feedback and its
  applications.
\newblock In \emph{Proc. ICML'14}, 901--909.

\bibitem[{Pukelsheim(2006)}]{pukelsheim2006}
Pukelsheim, F. 2006.
\newblock \emph{Optimal Design of Experiments}.
\newblock Society for Industrial and Applied Mathematics.

\bibitem[{Rejwan and Mansour(2020)}]{Idan2019}
Rejwan, I.; and Mansour, Y. 2020.
\newblock Top-$k$ Combinatorial Bandits with Full-Bandit Feedback.
\newblock In \emph{Proc. ALT'~20}, 752--776.

\bibitem[{Rusmevichientong and Williamson(2006)}]{Rusmevichientong2006}
Rusmevichientong, P.; and Williamson, D.~P. 2006.
\newblock An Adaptive Algorithm for Selecting Profitable Keywords for
  Search-based Advertising Services.
\newblock In \emph{Proc. EC '06}, 260--269.

\bibitem[{Soare, Lazaric, and Munos(2014)}]{Soare2014}
Soare, M.; Lazaric, A.; and Munos, R. 2014.
\newblock Best-Arm Identification in Linear Bandits.
\newblock In \emph{Proc. NIPS'14}, 828--836.

\bibitem[{Tao, Blanco, and Zhou(2018)}]{Tao2018}
Tao, C.; Blanco, S.; and Zhou, Y. 2018.
\newblock Best Arm Identification in Linear Bandits with Linear Dimension
  Dependency.
\newblock In \emph{Proc. ICML'18}, 4877--4886.

\bibitem[{Xu, Honda, and Sugiyama(2018)}]{Xu2018}
Xu, L.; Honda, J.; and Sugiyama, M. 2018.
\newblock A fully adaptive algorithm for pure exploration in linear bandits.
\newblock In \emph{Proc. AISTATS'18}, 843--851.

\bibitem[{Zaki, Mohan, and Gopalan(2019)}]{zaki2019towards}
Zaki, M.; Mohan, A.; and Gopalan, A. 2019.
\newblock Towards Optimal and Efficient Best Arm Identification in Linear
  Bandits.
\newblock \emph{arXiv preprint arXiv:1911.01695} .

\bibitem[{Zaki, Mohan, and Gopalan(2020)}]{Zaki+2020}
Zaki, M.; Mohan, A.; and Gopalan, A. 2020.
\newblock Explicit Best Arm Identification in Linear Bandits Using No-Regret
  Learners.
\newblock \emph{arXiv preprint arXiv:2006.07562} .

\bibitem[{Zhong, Cheung, and Tan(2020)}]{Zhong+2020}
Zhong, Z.; Cheung, W.~C.; and Tan, V.~Y. 2020.
\newblock Best Arm Identification for Cascading Bandits in the Fixed Confidence
  Setting.
\newblock \emph{to appear in ICML'20} .

\end{thebibliography}

\clearpage
\appendix
\section*{Appendix}
 \setcounter{equation}{0}
 
\section{Additional Related Work} \label{apx:related_work}

\begin{table*}[t]
	\centering
	\caption{Comparison between our results and existing results for CPE-PL (BL). ``General'' represents that the algorithm works for any combinatorial structure. $\tilde{O}(\cdot)$ only omits $\log \log$ factors. Main notations are defined in Section~\ref{sec:problem_def} and other specific notations are given in the footnote. 
	} \label{table:resultsfull}
	\renewcommand\arraystretch{1.8}
	\scalebox{0.75}{
	\begin{tabular}{|c|c|c|c|c|c|}
		\hline
		Algorithm&Sample complexity \footnotemark&Case&Problem Type&Strategy&Time\\
		\hline
		{\bf 
		\textsf{GCB-PE} (ours, Thm.~\ref{thm:GCB_ub})} &$O \big( \frac{ |\sigma| \beta_{\sigma}^2 L_p^2}{\Delta_{\textup{min}}^2} \log \frac{ \beta_{\sigma}^2 L_p^2}{\Delta_{\textup{min}}^2\delta} \big)$& General &CPE-PL&Static&$\mathrm{Poly}(d)$\\
		\hline
		{\bf  
		\textsf{PolyALBA} (ours, Thm.~\ref{thm:CPE-BL})} &$\tilde{O} \big(\sum_{i=2}^{\lfloor \frac{d}{2} \rfloor} \frac{1}{\Delta_i^2} \log \frac{|\mathcal{X}|}{\delta}    +\frac{d^2 m \xi_{\max}({\widetilde{M}({\lambda})}^{-1})}{ \Delta^2_{d+1}} \log \frac{|\mathcal{X}|}{\delta} \big)$&General &CPE-BL& Adaptive  &$\mathrm{Poly}(d)$\\
		\hline
		\textsf{ICB}~\cite{Kuroki+19}&$\tilde{O}\big(\frac{d  \xi_{\max}({M({\lambda})}^{-1})\rho(\lambda)}{\Delta^2_{\min}} \log \frac{d  \xi_{\max}({M({\lambda})}^{-1}) \rho(\lambda)}{\Delta^2_{\min}\delta}\big)$ &General &CPE-BL& Static&$\mathrm{Poly}(d)$\\
		\hline
		 \textsf{SAQM}~\cite{Kuroki+19} &$\tilde{O}\big(\frac{d^{1/4}k\xi_{\max}({M({\lambda})}^{-1})\rho(\lambda)}{\Delta_{\min}^2} \log \frac{d^{1/4}k\xi_{\max}({M({\lambda})}^{-1})\rho(\lambda)}{\Delta_{\min}^2  \delta}\big)$ &Top-$k$ &CPE-BL& Static&$\mathrm{Poly}(d)$ \\
		\hline
		\textsf{CSAR}~\cite{Idan2019} &$\tilde{O}\big(\sum^{d}_{i=2}\frac{1}{\tilde{\Delta}_i^2} \log\frac{d}{\delta}\big)$&Top-$k$ &CPE-BL&  Adaptive &$\mathrm{Poly}(d)$ \\
		\hline
		\textsf{$\mathcal{X}\mathcal{Y}$-static}~\cite{Soare2014}&$O\big(\frac{d}{\Delta_{\textup{min}}^2}\log \frac{|\cX|}{\delta \Delta^2_{\min}} +d^2\big)$&$\cX \subseteq \mathbb{R}^d$&BAI-LB& Static& $\Omega(|\cX|)$ \\
		\hline
		\textsf{\textsf{Explore-Verify}}~\cite{Karin16}&$O\big(\frac{d}{\Delta_{\textup{min}}^2} \log \frac{|\cX|}{\delta \Delta_{\min}}+ d\log \delta^{-1}\big)$&$\cX \subseteq \mathbb{R}^d$&BAI-LB& Static& $\Omega(|\cX|)$\\
		\hline
		\textsf{LinGapE}~\cite{Xu2018}&  $\tilde{O}\big(d  \sum_{x \in \cX} H_x\log \frac{d|\cX|}{\delta} \cdot \sum_{x \in \cX} H_x\big)$ &$\cX \subseteq \mathbb{R}^d$&BAI-LB& Adaptive & $\Omega(|\cX|)$\\
		\hline
		$\mathcal{Y}$-\textsf{ElimTil}~\cite{Tao2018}&$\tilde{O}\big(\frac{d}{\Delta_{\textup{min}}^2} (\log \delta^{-1} + \log |\mathcal{X}|)\big)$&$\cX \subseteq \mathbb{R}^d$&BAI-LB& Adaptive& $\Omega(|\cX|)$\\
		\hline
	    \textsf{ALBA}~\cite{Tao2018}&$\tilde{O}\big(\sum^{d}_{i=2}\frac{1}{\Delta_i^2}(\log \delta^{-1} + \log |\mathcal{X}|)\big)$&$\cX \subseteq \mathbb{R}^d$&BAI-LB& Adaptive & $\Omega(|\cX|)$\\
		\hline
		\textsf{RAGE}~\cite{Fiez2019}&$O \big(\sum^{\left \lfloor \log_2(4/\Delta_{\min})  \right \rfloor}_{t=1} 2(2^t)^2\tilde{\rho}(\mathcal{Y}(S_t)) \log(t^2|\mathcal{X}|^2/\delta) \big)$&$\cX \subseteq \mathbb{R}^d$&BAI-LB& Adaptive& $\Omega(|\cX|)$\\
		\hline
		\textsf{LinGame(-C)} \cite{Degenne+2020}&$\mathop{\lim \sup}_{\delta \rightarrow 0} \frac{\mathbb{E}_{\theta}[\tau_{\delta}]}{\log(1/ \delta)} \leq \min_{\lambda \in \triangle(\cX)}\max_{x\in \mathcal{\cX} \setminus\{x^*\}} \frac{2||x^*-x||_{M(\lambda)^{-1}}^2}{((x^*-x)^{\top}\theta)^2} $&$\cX \subseteq \mathbb{R}^d$&BAI-LB& Adaptive& $\Omega(|\cX|)$ \\
		\hline
		\textsf{Peace} \cite{Katz+2020}&$O\big( \big(\min_{\lambda \in \triangle(\cX)}\max_{x\in \mathcal{\cX} \setminus\{x^*\}} \frac{||x^*-x||_{M(\lambda)^{-1}}^2}{((x^*-x)^{\top}\theta)^2} + \gamma^*\big)  \log(1/ \delta)\big)$&$\cX \subseteq \mathbb{R}^d$&BAI-LB& Adaptive& $\Omega(|\cX|)$ \\
		\hline
		Lower Bound \cite{Fiez2019}&$\mathbb{E}_{\theta}[\tau_{\delta}] \geq \min_{\lambda \in \triangle(\cX)}\max_{x\in \mathcal{\cX} \setminus\{x^*\}} \frac{||x^*-x||_{M(\lambda)^{-1}}^2}{((x^*-x)^{\top}\theta)^2} \log(1/2.4 \delta)$&$\cX \subseteq \mathbb{R}^d$&BAI-LB& - & - \\
		\hline
	\end{tabular}
	}
\end{table*}

\footnotetext{
	Notations appearing in the table but not relevant in our problem setting are given below:
	$\rho(\lambda)=\max_{x \in \mathcal{X}}\|x\|^2_{M(\lambda)^{-1}}$.
	$\tilde{\Delta}_i=\theta_i - \theta_{k+1}$ if $i \leq k$ and $\theta_k-\theta_i$ otherwise.
	$H_x=\underset{x_i,x_j \in \cX}{\max}\frac{\bar{\rho}_x(x_i,x_j)}{\max\{\bar{\Delta}^2_i \bar{\Delta}^2_j\}}$  where $\bar{\Delta}=(x^*-x_i)^{\top}\theta$ if $x_i \neq x^*$, $\argmin_{x \in \cX} x^*-x$ otherwise, and $\bar{\rho}_x(x_i,x_j)$ is a term defined by the optimal solution to a convex optimization (see (11) in~\cite{Xu2018}).
	$S_t=\{x \in \cX \mid (x^*-x)^{\top}\theta\leq 4 \cdot 2^{-t} \}$. $\mathcal{Y}(S_t)=\{ x-x' \mid \forall x,x' \in S_t, x \neq x'\}$.
	$\tilde{\rho}(\mathcal{Y}(S_t))=\min_{\lambda \in \triangle(\cX)}\max_{v \in \mathcal{Y}(S_t) }\|v\|_{\Mlin}$. 
	$\gamma^*=\min_{\lambda \in \triangle(\cX)} \mathbb{E}_{\eta \sim N(0,1)} \left[ \max_{x\in \mathcal{\cX} \setminus\{x^*\}} \frac{(x^*-x)^\top M(\lambda)^{-1/2} \eta}{(x^*-x)^{\top}\theta} \right]^2. $
}

In this section, we further review the full-literature of BAI-LB and variants of CPE recently proposed. Comparison between our algorithms and existing algorithms is summarized in Table~\ref{table:results}.

\paragraph{Best Arm Identification in Linear Bandits (BAI-LB).}
Soare et al.~\shortcite{Soare2014} addressed the BAI-LB in the fixed confidence setting and first provided the static allocation algorithm for BAI-LB by introducing the interesting connection between BAI-LB and \emph{G-optimal experimental design}~\cite{pukelsheim2006}.
Tao et al.~\shortcite{Tao2018} analyzed the novel randomized estimator based on the convex relaxation of G-optimal design, and devised the adaptive algorithm whose sample complexity depends linearly on the dimension $d$.
Xu et al.~\shortcite{Xu2018} proposed a fully adaptive algorithm inspired by UGapE~\cite{Gabillon2012}.
Karnin~~\shortcite{Karin16} analyzed the explore-verify algorithms for several settings of BAI including linear, dueling, unimodal, and graphical bandits.
Fiez et al.~\shortcite{Fiez2019} introduced the transductive BAI-LB, and proposed the first non-asymptotic algorithm that nearly achieves the information-theoretic lower bound.
Zaki et al.~\shortcite{zaki2019towards} proposed a generalized LUCB algorithm for BAI-LB with sample complexity analysis for the special cases of two and three arms.
Degenne et al.~\shortcite{Degenne+2020} designed the first asymptotically optimal sampling rules for BAI-LB.
Katz-Samuels et al.~\shortcite{Katz+2020} proposed near-optimal algorithms for both fixed confidence and fixed budget settings by leveraging the theory of suprema of empirical processes.
Although they further discussed a computationally efficient algorithm for CPE-MB with linear rewards, it cannot be applied to CPE-BL since the set of measurement vectors considered in \cite{Katz+2020} is also exponentially large for CPE-BL.
Zaki, Moha, and Gopalan~\shortcite{Zaki+2020} proposed the explicitly described algorithm by using tools from game theory and no-regret learning to solve minimax games.

Despite the recent advances in BAI-LB described above, no existing algorithms can  solve CPE-BL efficiently, since their computation for distribution $\lambda_*$ and  ways to maintain alive action set
 cost exponential time in the settings of CPE-BL and CPE-PL. Moreover, no existing algorithms can deal with the nonlinear reward functions or partial linear feedback in CPE-PL.

\paragraph{Combinatorial Pure Exploration (CPE).}
For variants in CPE, Huang et al.~\shortcite{huang2018} designed the algorithm for CPE with continuous and separable reward functions. Although their model considers nonlinear rewards, it cannot deal with either full-bandit or partial linear feedback and thus cannot be applied to our settings. Recently, the best arm identification in cascading bandits is investigated by Zhong et al.~\shortcite{Zhong+2020}. Their problem setting can be seen as another model of top-$k$ identification with partial feedback.
However, their algorithm is designed for identifying the top-$k$ actions and thus it cannot be applied to the case under other combinatorial structures such as paths, mathchings, and matroids.
Kuroki et al.~\shortcite{Kuroki2020} studied a special sub-problem of CPE-PL, where the offline optimization is the densest subgraph problem and the learner observes full-bandit feedback for a set of edges. Their algorithms and analysis use the property of the average degree, and thus they cannot be directly applied to solve either CPE-BL or CPE-PL. 
Chen et al.~\shortcite{chen2020cpe_db} studied an adaptation of CPE to the dueling bandit setting, where at each timestep the leaner plays a pair of edges (base arms) in a bipartite graph and observes a random outcome of the comparison  with the objective of identifying the optimal matching. They only consider relative feedback between two compared base arms in the matching case, and thus their algorithms cannot be applied to either CPE-BL or CPE-PL.    

\section{Illustration Examples for CPE-PL}\label{apx:model_example}

\begin{figure}[t]
	\centering
	\includegraphics[keepaspectratio,width=40mm]
	{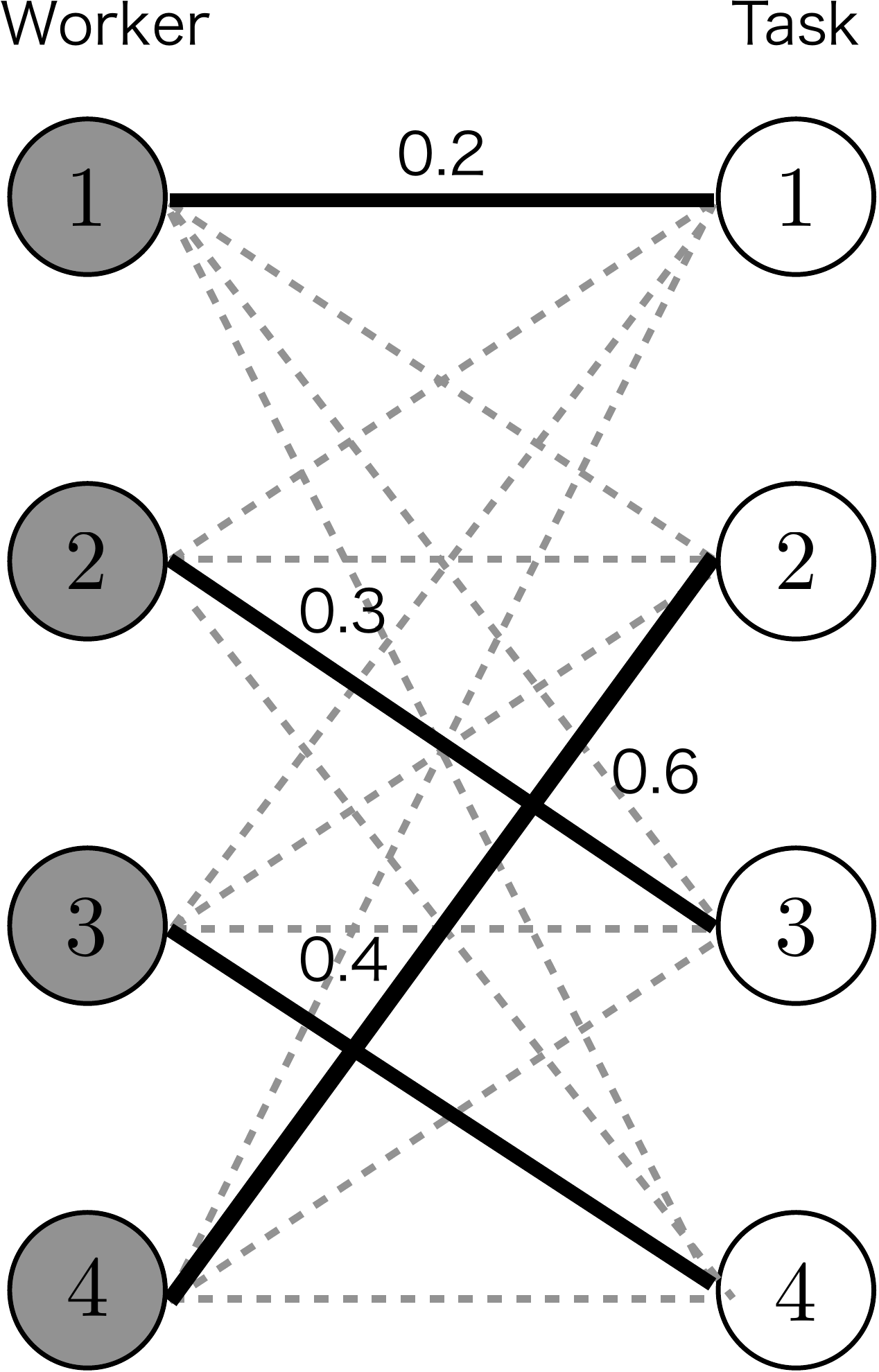}\caption{Example of the crowdsourcing application.}\label{Fig:crowd}
\end{figure}

In this section, we will provide specific examples to illustrate the feedback model of CPE-PL by considering the crowdsourcing scenario (see Figure~\ref{Fig:crowd}).
There are $N=4$ workers and $M=4$ tasks, which can be represented as a complete bipartite graph $G=(L, R, E)$.
In this bipartite graph, each edge corresponds to a pair between a worker and a task, and its edge-weight corresponds to the utility that is unknown to the learner.
In the model of CPE-PL, each base arm $i \in \{1,2,...,d \}$ prescribes each edge $e \in E$ where $d=|E|$ and its mean is the unknown edge-weight $\theta_e$.
Let edges $(1, 1), (2, 3), (3, 4), (4, 2)$ be the first, second, third, fourth base arms respectively.
Suppose that random outcome at round $t$ is $\theta+\eta_t=(0.2, 0.3, 0.4, 0.6, \ldots)^\top \in \mathbb{R}^d$, and the learner will pull the action $x_t=(1, 1, 1, 1 ,0, 0,\ldots,0)^\top$.
In this case, the examples of transformation matrix $M_{x} \in \mathbb{R}^{m_x \times d}$ and corresponding random feedback $y_t=M_{x_t}(\theta+\eta_t) \in \mathbb{R}^{m_{x_t}}$ can be written as follows.

If $M_{x_t}=(1, 0, 0, \ldots,0)$, the learner can observe only first base arm, that is,
\[
M_{x_t}(\theta+\eta_t)=(1,0,0,\ldots,0) \left(
     \begin{array}{c}
       0.2 \\
       0.3\\
       0.4 \\
       0.6 \\
       \vdots
     \end{array}
   \right)=0.2=y_t.
\]
If $M_{x_t}=\mathrm{diag}(x)$, the learner can obtain semi-bandit feedback, that is,
\begin{align*}
	M_{x_t}(\theta+\eta_t)=&\left(
	\begin{array}{ccccc}
	1 & 0 & 0 &0 &\cdots \\
	0 & 1 & 0 & 0& \cdots\\
	0 & 0 & 1 & 0 & \cdots \\
	0 & 0 & 0 & 1 & \cdots
	\end{array}
	\right) \left(
	\begin{array}{c}
	0.2 \\
	0.3\\
	0.4 \\
	0.6 \\
	\vdots
	\end{array}
	\right)
	\\
	=& \left(
	\begin{array}{c}
	0.2 \\
	0.3\\
	0.4 \\
	0.6 
	\end{array}
	\right)
	\\
	=&y_t.
\end{align*}
If $M_{x_t}=x_t^\top$, the learner can only observe the sum of the rewards (i.e., full-bandit feedback), that is,
\[
M_{x_t}(\theta+\eta_t)=(1,1,1,1,0,\ldots,0)  \left(
     \begin{array}{c}
       0.2 \\
       0.3\\
       0.4 \\
       0.6 \\
       \vdots
     \end{array}
   \right)= 1.5=y_t.
\]
For another examples, the learner may obtain the sum of the rewards from two base arms as follows,
\begin{align*}
	M_{x_t}(\theta+\eta_t)=&\left(
	\begin{array}{ccccc}
	1 & 1 & 0 &0 &\cdots \\
	0 & 0 & 1 & 1& \cdots
	\end{array}
	\right) \left(
	\begin{array}{c}
	0.2 \\
	0.3\\
	0.4 \\
	0.6 \\
	\vdots
	\end{array}
	\right)
	\\
	=& \left(
	\begin{array}{c}
	0.5 \\
	1.0
	\end{array}
	\right)
	\\
	=&y_t.
\end{align*}

Note that reward functions in CPE-PL can be nonlinear, while feedback model is linear as illustrated above.



\section{More Discussion on the Optimality of \textsf{PolyALBA}} \label{apx:polyalba_optimality}
For BAI-LB, \citet{Fiez2019} propose the first sample complexity lower bound 
\begin{align*}
	\mathbb{E}_{\theta}[\tau_{\delta}] \geq \min_{\lambda \in \triangle(\cX)}\max_{x\in \mathcal{\cX} \setminus\{x^*\}} \frac{||x^*-x||_{M(\lambda)^{-1}}^2}{((x^*-x)^{\top}\theta)^2} \log(1/2.4 \delta), 
\end{align*}
and a nearly optimal algorithm \textsf{RAGE} with sample complexity
\begin{align*}
	O \left( \sum^{\left \lfloor \log_2(4/\Delta_{\min})  \right \rfloor}_{t=1} 2(2^t)^2\tilde{\rho}(\mathcal{Y}(S_t)) \log(t^2|\mathcal{X}|^2/\delta) \right),
\end{align*}
where $S_t=\{x \in \cX \mid (x^*-x)^{\top}\theta\leq 4 \cdot 2^{-t} \}$, $\mathcal{Y}(S_t)=\{ x-x' \mid \forall x,x' \in S_t, x \neq x'\}$ and $\tilde{\rho}(\mathcal{Y}(S_t))=\min_{\lambda \in \triangle(\cX)}\max_{v \in \mathcal{Y}(S_t) }\|v\|_{\Mlin}$. \citet{Fiez2019} prove that this upper bound for \textsf{RAGE} matches the lower bound  with a logarithmic factor. 

Below we show that for a family of instances, the sample complexity of \textsf{PolyALBA} matches that of \textsf{RAGE} and also achieves near-optimality.
Consider a family of instances such that $\Delta_{\left \lfloor d/ 2^{(t-2)} \right \rfloor+1}=4\cdot2^{-t}, \  t=2, 3, \dots, \log_2(\frac{4}{\Delta_{\min}})$,
which can be easily found in practical sub-problems. For example, in the \emph{Multi-Bandit} case~\cite{Gabillon+2011_multibandit}, we are given base arms numbered by $1,\dots,10$, and a feasible super arm consists of two base arms respectively from $\{1,\dots,5\}$ and $\{6,\dots,7\}$, where $d=10$ and $|\cX|=25$. Let $\theta=( 2.5, 2, 1.5, 1, 0.5, 0.625, 0.5, 0.375, 0.25, 0.125 )^\top$. 
We use $x_i$ and $A_i$ to denote the indicator vector and the set of base arm indices for the $i$-th best super arm.
Then, $A_1=\{1,6\}, x_1^\top \theta=3.125, A_2=\{1,7\}, x_2^\top \theta=3, A_3=\{1,8\}, x_3^\top \theta=2.875, A_6=\{2,6\}, x_6^\top \theta=2.625, A_{11}=\{3,6\}, x_{11}^\top \theta=2.125$, which belongs to the considered family of instances.

In such family of instances, our \textsf{PolyALBA} performs similar elimination procedure as \textsf{RAGE}. Specifically, since the sample complexity upper bound of \textsf{PolyALBA} guarantees that for epoch $q$, any action $x$ with $\Delta_x \geq \Delta_{\left \lfloor d/ 2^q \right \rfloor +1}$ will be discarded, we have that $\forall x' \in S_q, \Delta_{x'} \leq \Delta_{\left \lfloor d/ 2^{q-1} \right \rfloor +1}=2 \cdot 2^{-q}, \  q=1, 2, \dots, \log_2(\frac{4}{\Delta_{\min}})$, which is the same to the gap constraint for $S_t$ in \textsf{RAGE}. 
Then, in epoch $q= 2, \dots, \log_2(\frac{4}{\Delta_{\min}})$, \textsf{PolyALBA}  essentially performs the same procedure as \textsf{RAGE}, while in epoch $q=1$, \textsf{PolyALBA} performs a higher samples (the $\xi_{\max}({\tilde{M}^{-1}({\lambda})})$-related term) than \textsf{RAGE}. Formally, in such instances \textsf{PolyALBA} has sample complexity
\begin{align*}
	O \bigg( & \sum^{\left \lfloor \log_2(4/\Delta_{\min})  \right \rfloor}_{t=2} 2(2^t)^2\tilde{\rho}(\mathcal{Y}(S_t)) \log(t^2|\mathcal{X}|^2/\delta) \\& + m d  \xi_{\max}({\tilde{M}^{-1}({\lambda})}) \tilde{\rho}(\mathcal{Y}(S_1)) \log(|\mathcal{X}|^2/\delta) \bigg).
\end{align*}
For sufficiently small $\Delta_{\min}$ (increase $d$ in the above Multi-Bandit case to obtain small $\Delta_{\min}$), 
the additional term related to $\xi_{\max}({\tilde{M}^{-1}({\lambda})})$ is absorbed and the result is the same to that of \textsf{RAGE}, which matches the lower bound within a logarithmic factor.
The sample complexity of \textsf{PolyALBA} shows superiority over other heavily $\Delta_{\min}$-dependent algorithms~\cite{Soare2014,Karin16,Kuroki+19}, and the additional term related to $\xi_{\max}({\tilde{M}^{-1}({\lambda})})$ can be viewed as a cost for achieving computational efficiency.

We also remark that \textsf{PolyALBA} is close to the lower bound in the worst-case instances and our sample complexity has an interesting relation with that of G-allocation strategy for BAI-LB \cite{Soare2014}.
\citet{Soare2014} proved the following lemma to bound the \emph{oracle complexity} $H_{\mathrm{LB}}$ defined as follows, which also appears in the information theoretic lower bound~\cite{Fiez2019}.
\[
H_{LB}=\min_{\lambda \in \triangle(\cX)} \max_{x\in \mathcal{\cX} \setminus\{x^*\}}\frac{ \|x^*-x\|^2_{M(\lambda)^{-1}}}{((x^*-x)^{\top}\theta)^2}.
\]
\begin{lemma}[Lemma 2, ~\citet{Soare2014}]
    Given an arm set $\cX \subseteq \mathbb{R}^d$ and parameter $\theta$, the complexity $H_{\mathrm{LB}}$ is such that
    \[
    \max_{x\in \mathcal{\cX} \setminus\{x^*\}}\frac{ \|x^*-x\|^2_2}{L_x \Delta^2_{\min}}\leq H_{\mathrm{LB}} \leq \frac{4d}{\Delta_{\min}^2},
    \]
    where $L_x$ is the upper bound of the $\ell_2$-norm of any $x \in \cX$.
    Furthermore, if $\cX$ is the canonical basis, the problem reduces to a MAB and $\sum_{i=1}^{|\cX|}1/\Delta_i \leq H_{\mathrm{LB}} \leq 2\sum_{i=1}^{|\cX|}1/\Delta_i$.
\end{lemma}\citet{Soare2014} designed the G-allocation strategy and its sample complexity is $O\left(\frac{4d}{\Delta_{\min}^2}\log\frac{|\cX|}{\delta}\right)$, which shows that G-allocation strategy is optimal within logarithmic terms for instances where the worst-case value of  $H_{\mathrm{LB}}$ is given (See Theorem 1 in their paper).
For instances with $\Delta_i=\Delta_{\textup{min}}, \forall i>1$, 
\textsf{PolyALBA} has the sample complexity of $\tilde{O}\left(\frac{md^2 \xi_{\max}({\tilde{M}^{-1}({\lambda})})}{\Delta^2_{\min}}\log \frac{|\cX|}{\delta} \right)$, which matches
that of G-allocation strategy  up to a factor of $md\xi_{\max}({\tilde{M}^{-1}({\lambda})}))$ (when we ignore the $\log \log$ terms).
\textsf{PolyALBA} also employs G-optimal design for its static phase but we restrict its support in order to make it running in polynomial time, which causes the additional term related to $\xi_{\max}({\tilde{M}^{-1}({\lambda})})$.
Note that the lower bound~\cite{Fiez2019} does not consider time complexity, and the gap in the additional term may be a price paid for achieving computational efficiency.

To sum up, to our best knowledge, \textsf{PolyALBA} is the first polynomial-time adaptive algorithm that works for CPE-BL with general combinatorial structures and achieves nearly optimal sample complexity for a family of problem instances.
It is very interesting to study a lower bound for polynomial-time CPE-BL algorithms and it remains open to design efficient algorithms with instance-wise optimal sample complexity.



\section{Naive Reduction of Top-$k$ and Analysis of a UCB-based Algorithm for CPE-BL}\label{apx:clucb}

 \begin{algorithm}[ht]
  \caption{Combinatorial lower-upper naive confidence bound (\textsf{CLUNCB})}\label{alg:CLUCB}%
	\SetKwInOut{Input}{Input}
	\SetKwInOut{Output}{Output}
	\Input{ Accuracy $\epsilon>0$, confidence level $\delta \in (0,1)$}
	
    \textbf{Initialization}  For each $e \in [d]$, pull $x_e \in \cX$  such that $e \in x_e$ once. Initialize $\Atl$ and $b_t$;
    
	\While{ $\widetilde{\theta}_t^\top \widetilde{x}_t-\widetilde{\theta}^{\top}\widehat{x}^*_t \leq \epsilon$ is not true}{
	
            $t \leftarrow t+1$;
    
            $\widehat{x}^*_t   \leftarrow  \argmax_{x \in {\cX}} \widehat{\theta}_t^{\top} x$;
        	
            Set ${\rm rad}_t(e)=C_t \sqrt{\Atinl(e,e)} $ for all $e \in [d]$;
            

            \For{$e=1,\ldots, d$ }{
            	\textbf{if} $e \in \widehat{x}^*_t   $ \textbf{then} $\widetilde{\theta}_t(e) \leftarrow \thetahat(e)-{\rm rad}_t(e)$;
            	
                \textbf{else} $\widetilde{\theta}_t(e) \leftarrow \thetahat(e) +{\rm rad}_t(e)$;
            }

             $\widetilde{x}_t \leftarrow  \argmax_{x \in \cX} \widetilde{\theta}_t^{\top}x $;

            $p_t  \leftarrow \argmax_{e \in  (\widetilde{x}_t \setminus \widehat{x}^*_t)   \bigcup (\widehat{x}^*_t \setminus \widetilde{x}_t)} {\rm rad}_t(e)$;
            
            
            
            Sample any $x_t \in \cX$ such that $p_t \in x_t$; 

            Update $\Atl$, $b_t$ and $\widehat{\theta}_t$;


    }
    \textbf{Return }{$\texttt{Out} \leftarrow \widehat{x}^*_t$}
\end{algorithm}

In this section,
we briefly explain a naive reduction to the classic CPE-MB
for CPE-BL in the top-$k$ setting, and note that there is no simple reduction for general CPE-BL by showing an undesirable property of a UCB-based algorithm with a regularized least-square estimator.

We first remark that with only $O(k)$ more samples, the problem of top-$k$ identification with full-bandit feedback can be solved by classic top-$k$ algorithms in which base arms are queried.
Suppose that an algorithm has a sample complexity of $C_{\delta, \Delta}$ in classic setting, it yields a complexity of $\tilde{O}(k \cdot C_{\delta, \Delta} )$ for the full-bandit setting, where $\tilde{O}$ omits some log factors.
This is due to the fact that the unbiased estimate can be obtained for the difference between the two base arms by comparing two $k$-base arm queries with one base arm difference.
Formally, with any fixed base arm $i_0$, one can get an unbiased estimate for the gap $\theta_j -\theta_{i_0}$ with $O(k)$ times larger variance by querying two super-arms $S \cup \{j\}$ and $S \cup \{i_0\}$ for $S \subseteq [d]$ such that $|S|=k-1$ and $j, i_0 \notin S$.
Therefore, when $C_{\delta, \Delta}$ has dependence of $\sum_{i\in [d]} \Delta^{-2}_i$, we also have a sample complexity which has dependence of $\sum_{i\in [d]} \Delta^{-2}_i$for top-$k$ case with full-bandit feedback.
However,
for more complex cases such as matroid, matroid intersection,
and $s$-$t$ path,
we cannot use such a reduction due to its combinatorial constraint.

We show that with a simple modification using the regularized least-square estimator,
CLUCB algorithm~proposed in~\cite{Chen2014} can work for CPE-BL for general constraints such as top-$k$, matroid, matroid intersection, $s$-$t$ path, and it is a polynomial-time $(\epsilon,\delta)$-PAC algorithm.
However, we prove that this naive adaption can be sub-optimal and the sample complexity depends on $\frac{1}{\Delta_{\min}^2}$ in the worst case.

\subsection{Preliminary}
We call an algorithm \emph{fully adaptive} if it changes the arm-selection strategy based on the past observation at all rounds. For such an adaptive algorithm, we cannot use the ordinary least-square estimator for $\theta \in \mathbb{R}^d$ as an unbiased estimator. Instead we will use the regularized least-square estimator.
If the sequence of super-arm selections $\textbf{x}_t= (x_1, \ldots, x_t)$ is adaptively determined based on the past observations, the regularized least-square estimator is given by
\begin{align}\label{OLS}
    \thetahat=\Atinl b_{{\bf x}_t},
\end{align}
where $\Atl$ and $b_{{\bf x}_t}$ is defined by 
\[
\Atl=\iota I+\sum_{i=1}^t x_i x_i^{\top},
\ \ \text{and} \ \
b_{{\bf x}_t}=\sum_{i=1}^tx_ir_i \in \mathbb{R}^n.
\]
for regularization parameter $\iota>0$ and the identity matrix $I$.
If we set $\iota=0$ and we are allowed to sample a base arm, i.e., unit vector at all rounds, it is easy to see that $\Atl(i,i)=T_i(t)$ for $i \in [d]$ and $\Atl(i,j)=0$ for $i \neq j$,
where $T_i(t)$ is the number of times that base arm $i$ is sampled before round $t+1$.

Abbasi-Yaddkori et al.~\shortcite{yadkori11} showed the high probability bound for the regularized least-squares estimator $\hat{\theta}$.
\begin{proposition}[Theorem~2 in Abbasi-Yaddkori et al.~\shortcite{yadkori11}]\label{proposi:Abbasi}
Let $\hat{\theta}_t$ be the regularized least-squares estimator.
Suppose that a noise $\eta_t$ is $\kappa$-sub-Gaussian.
If the $\ell_2$-norm of parameter $\theta$ is less than $L$,
then for all $i \in [d]$ and for every adaptive sequence $\bf{x}_t$,
\begin{align*}
|x^{\top}\theta-x^{\top} \widehat{\theta}_t| \leq C_t \|x\|_{\Atinl}
\end{align*}
holds for all $t \in \{ 1,2,\ldots\}$
and $\forall x \in \mathbb{R}^n$ with probability at least $1-\delta$,
where 
\begin{align}\label{def:ct}
    C_t =\kappa \sqrt{2\log \frac{ \det(\Atl)^{\frac{1}{2}}}{\iota^{\frac{n}{2}}}} +\iota^{\frac{1}{2}}L.
\end{align}
Moreover, if $\| x \|_2 \leq \sqrt{m}$ holds for all $t>0$,
then
\begin{align}\label{ineq:ct}
    C_t \leq \kappa \sqrt{d\log \frac{ 1+tm/\iota}{\delta}} +\iota^{\frac{1}{2}}L.
\end{align}
\end{proposition}





We also introduce the notion of the \emph{width} for a decision set $\cX$ defined in~\citet{Chen2014}; $\mathrm{width}(\cX)$ prescribes the size of the thinnest exchange class (see \citet{Chen2014} for detailed definition). For example, if $\cX$ are independent sets of ground set $[d]$, $\mathrm{width}(\cX) \leq 2$. 
\subsection{Analysis of CLUCB for CPE-BL}

 In the setting where each base arm is pulled at all rounds,
the confidence radius is simply defined as ${\rm rad}_t(e)=\sqrt{\frac{2\log\left(\frac{4dt^3}{\delta} \right) }{T_e(t)} } $ for all $e \in [d]$. Since we are not allowed to pull each base arm, we cannot define such a radius as the above form. However, we have concentration inequalities for each unit vector of $e$, and thus we can construct the confidence radius in the full-bandit setting.
From Proposition~\ref{proposi:Abbasi},
we can construct the high probability confidence radius as follows.

\begin{lemma}\label{lemma:concentration_ineq}
Suppose that a reward from each base arm follows a $1$-sub-Gaussian distribution for all $i \in [d]$.
For all $t>0$ and all $i \in [d]$, the confidence radius $\mathrm{rad}_t(i)$ is defined as
\begin{align}\label{eq:radius}
 \mathrm{rad}_t(i)=C_t \sqrt{\Atinl(i,i)} \ \ \ \ (\forall i \in [d]),
\end{align}
where $C_t$ is given by \eqref{def:ct}.
Let $\bold{rad}_t$ be an $d$-dimensional vector with nonnegative entries.
For $\bold{rad}_t$, define random event $\cE_t$ for all $t>0$ as follows.
\begin{align}\label{event}
\cE_t= \{  \forall i \in [d],\ |\theta(i)-\hat{\theta}_t(i)| \leq \mathrm{rad}_t(i) \} 
\end{align}
Then we have 
\begin{align}
    \Prob \left[ \bigcap_{t=1}^{\infty}\cE_t \right] \geq 1-\delta.
\end{align}
\end{lemma}
The proof is omitted since it is straightforward from Proposition~\ref{proposi:Abbasi} and union bounds.
Using the above confidence radius, we can design CLUCB-based algorithm for CPE-BL, which is detailed in Algorithm~\ref{alg:CLUCB}.
We show that Algorithm~\ref{alg:CLUCB} is $(\epsilon,\delta)$-PAC and its sample complexity bound is given in Corollary~\ref{Corollary:CLUCB}.
As can be seen, the sample complexity depends on $\Delta^{-2}_{\min}$ in the worst case.
Also, since \textsf{CLUNCB} is fully adaptive, 
we cannot completely control $\lambda_{\mathrm{C}}$ beforehand and thus $M(\lambda_{\mathrm{C}})^{-1}(e,e) \leq \lambda_{\max}(M(\lambda_{\mathrm{C}})^{-1})$ can be large.

\begin{corollary}\label{Corollary:CLUCB}
 Let $\lambda_{\mathrm{C}} \in \triangle(\cX)$ be a distribution in which $\lambda_{\mathrm{C}}(x)$ represents the ratio that $x$ is pulled by \textsf{CLUNCB}.
The total number of samples $T$ is bounded as
\begin{align*}
    T  
    \!=\!  O  \left(  d \kappa^2  \tilde{H} \log  \left(  \frac{ d \kappa^2  m \tilde{H}/\iota   \!+\!   \log \delta^{-1}}{\delta}  \right)  \right)  \!,
\end{align*}
where $\tilde{H}$ is defined as 
\begin{align*}
  \tilde{H}  \!=\!  \max_{e \in [d]} \left( M(\lambda_{\mathrm{C}})^{-1} (e,e) \min  \left\{  \frac{9 \mathrm{width}(\cX)^2 }{\Delta^2_e} ,   \frac{4m^2}{ \epsilon d^2} \right\}\right) \!.
\end{align*}

\end{corollary}
\begin{proof}

First, we state the following two lemmas in Chen et al.~\shortcite{Chen2014};
Lemma~\ref{lemma:lemma12_chen_etal14} (Lemma 12 in Chen et al.~\shortcite{Chen2014}) shows that if the confidence radius is valid, then CLUCB always outputs $\epsilon$-optimal set,
and
Lemma~\ref{lemma:lemma13_chen_etal14} (Lemma 13 in Chen et al.~\shortcite{Chen2014}) implies that if the confidence radius of an arm is sufficiently small, then the arm will not be chosen as $p_t$.
Note that we have the lemmas since Lemmas~3, 5, 7, and 10 in Chen et al.~\shortcite{Chen2014} also hold for our setting where $\bold{rad}_t$ is given by \eqref{eq:radius} and the empirical mean is replaced with the least square estimator $\hat{\theta}_t$ in~\eqref{OLS}.

\begin{lemma}[Lemma 12 in Chen et al.~\shortcite{Chen2014}]\label{lemma:lemma12_chen_etal14}
If CLUCB stops on round $t$ and suppose that $\cE_t$ occurs. Then, we have $\theta^{\top}x^*-\theta^{\top}x_{\textsf{Out}}\leq \epsilon$.
\end{lemma}

\begin{lemma}[Lemma 13 in Chen et al.~\shortcite{Chen2014}]\label{lemma:lemma13_chen_etal14}
Given any $t$ and suppose that event $\cE_t$ occurs.
For any $e \in [d]$,
if $\mathrm{rad}_t(e) <\max\left\{ \frac{\Delta_e}{ 3\mathrm{width}(\cX)}, \frac{\epsilon}{2m} \right\}$, then $p_t \neq e$.
\end{lemma}

The random event $\bigcap_{t=1}^{\infty}\cE_t$ occurs with probability at least $1-\delta$ from 
Lemma~\ref{lemma:concentration_ineq}.
From Lemma~\ref{lemma:lemma12_chen_etal14},
under the event $\bigcap_{t=1}^{\infty}\cE_t$,
\textsf{CLUNCB} returns an $\epsilon$-optimal set.
Therefore,
in the rest of part, we shall assume this event holds.

Fix any arm $e \in [d]$ and let $\tau_e$ be the last round which arm $e$
 is chosen as $p_t$.
From \eqref{ineq:ct} and for a small $\iota \leq \frac{\kappa^2}{L^2} \log \delta^{-1}$,
we have 
\begin{align*}
     C_{\tau_e} \leq & \kappa \sqrt{n\log \frac{ 1+\tau_e m/\iota}{\delta}} +\iota^{\frac{1}{2}}L 
     \\
     \leq & 2\kappa \sqrt{d \log \left( \frac{1+\tau_e m/\iota }{ \delta} \right) }.
\end{align*}
By Lemma~\ref{lemma:lemma13_chen_etal14}, we have $\mathrm{rad}_{\tau_e}(e) \geq \max\left\{ \frac{\Delta_e}{ 3\mathrm{width}(\cX)}, \frac{\epsilon}{2m} \right\}$.
We define  $\Lambda_{\bf{x}_{\tau_e}} = \frac{\Atl}{\tau_e}$.
Then, we have
 \begin{align*}
 &\max\left\{ \frac{\Delta^2_e}{ 9\mathrm{width}(\cX)^2}, \frac{\epsilon^2}{4m^2} \right\} \leq \mathrm{rad}^2_t(e)\\
 &= C_{\tau_e}^2 \Atauinl(e,e) \\
 & \leq 4 \kappa^2 d \log \left( \frac{1+\tau_e m/\iota }{ \delta} \right)   \Atauinl(e,e)\\
 & = 4 \kappa^2 d \log \left( \frac{1+\tau_e m/\iota }{ \delta} \right)   \frac{\Lamdatauinl(e,e)}{\tau_e}.
 \end{align*}
 That is, we obtain
 \begin{align*}
     \tau_e \leq
     H_e \log \left( \frac{1+\tau_e m/\iota }{ \delta} \right),
 \end{align*}
where we define $H_e= 4 \kappa^2 d  \Lamdatauinl(e,e) 
     \min\left\{ \frac{9 \mathrm{width}(\cX)^2 }{\Delta^2_e},  \frac{4m^2}{ \epsilon^2} \right\}$.
 Let $\tau' (\leq \tau_e)$ satisfying 
 \begin{align}\label{eq:tau}
     \tau_e = H_e \left( \log \left( 1+\frac{\tau' m}{\iota } \right)  +\log \frac{1}{\delta} \right).
 \end{align}
 Then, we have
  \begin{alignat}{4}\label{ineq:tau'}
 \tau' \leq \tau_e &  =H_e  \left( \log \left( 1+\frac{\tau' m}{\iota} \right) + \log\frac{1}{\delta}   \right) \notag \\
 & \leq  H_e  \left( \sqrt{ \frac{\tau' m}{\iota}}  + \log\frac{1}{\delta}   \right) .
 \end{alignat}

 Solving \eqref{ineq:tau'} for $\sqrt{\tau'}$,
 we obtain 
   \begin{align*}
 \sqrt{\tau'} \leq&  \frac{1}{2}\left( H_e \sqrt{\frac{m}{\iota}}+\sqrt{H_e^2 \frac{m}{\iota}+4 H_e \log \frac{1}{\delta}} \right) \\
 & \leq 2 \sqrt{H_e^2 \frac{m}{4\iota}+ H_e \log \frac{1}{\delta}}.
 \end{align*}
 That is, we see that 
 $\tau'=O \left( H_e^2 \frac{m}{\iota}+ H_e \log \frac{1}{\delta} \right)$,
 which shows that  
 \begin{align}\label{order:tau'}
   \log \left( \frac{1+ \tau'm}{\iota} \right)=
   O \left( \log \left( \frac{ m H_e}{\iota} +\log \frac{1}{\delta} \right) \right).
 \end{align}
Combining \eqref{order:tau'} into \eqref{eq:tau},
we obtain
\begin{align*}
    \tau_e =   O \left(H_e \log \left( \frac{ m H_e/\iota+\log \delta^{-1}}{\delta}\right) \right).
\end{align*}
The number of samples used by \textsf{CLUNCB} is $T = \max_{e \in [d]} \tau_e$.

Recall that $\lambda_{\mathrm{C}} \in \triangle(\cX)$ is a distribution in which $\lambda_{\mathrm{C}}(x)$ represents the ratio that $x$ was pulled by \textsf{CLUNCB}.
Suppose that $T$ is sufficiently large such that $\Lambda_{\bf{x}_T}=\frac{\Atl}{T} \approx M(\lambda_{\mathrm{C}})$.
Define 
\begin{align*}
  \tilde{H}= \max_{e \in [d]} \left( (M(\lambda_{\mathrm{C}})^{-1}(e,e) \min\left\{ \frac{9 \mathrm{width}(\cX)^2 }{\Delta^2_e},  \frac{4m^2}{ \epsilon d^2} \right\}\right) 
\end{align*}

Then,
we have 
\begin{align*}
    T = \max_{e \in [d]} \tau_e
   = O \left( d \kappa^2  \tilde{H} \log \left( \frac{ d \kappa^2  m\tilde{H} /\iota+\log \delta^{-1}}{\delta}\right) \right).
\end{align*}
\end{proof}

\section{Entropic Mirror Descent Algorithm for Computing the Distribution $\lambda^*_{\mathcal{X}_{\sigma}}$}\label{apx:G-opt}

For completeness, we provide the algorithm for
	computing the optimal distribution $\lambda^*_{\mathcal{X}_{\sigma}} = \min_{\lambda \in \triangle(\mathcal{X}_{\sigma})}\max_{x \in \mathcal{X}_{\sigma}} x^{\top} M(\lambda)^{-1}x$ 
	for a given set of actions $\mathcal{X}_{\sigma} \subseteq {\mathcal{X}}$, which is used in both \textsf{ALBA}~\cite{Tao2018} and our method \textsf{PolyALBA}.
	Algorithm~\ref{alg:descent} details the entropic mirror descent algorithm proposed in \citet{Tao2018}.

\begin{algorithm}[ht!]
	\caption{The entropic mirror descent for computing $\lambda^*_{\mathcal{X}_{\sigma}}$~\cite{Tao2018} }\label{alg:descent}
	\SetKwInOut{Input}{Input}
	\SetKwInOut{Output}{Output}
	\Input{$d$-set of base arms $[d]$, a set of super arms $\mathcal{X}_{\sigma} \subseteq {\mathcal{X}}$, Lipschitz constant $L_f$ of function $\log \det M(\lambda)$ and tolerance $\epsilon$}

	Choose , such that  $\mathrm{rank}(X)=d$ where $X=(b_1, \dots, b_d)$\;
	
	Initialize $t \leftarrow 1$ and $\lambda^{(1)} \leftarrow (1/|\mathcal{X}_{\sigma}|,\ldots, 1/|\mathcal{X}_{\sigma}|)$;
	
	\While{$|\max_{x \in \cX_{\sigma}} x^{\top} M(\lambda^{(t)})^{-1}x|-d \geq	\epsilon$}{
	$a_t \leftarrow \frac{\sqrt{2 \ln |\mathcal{X}_{\sigma}| }}{L_f \sqrt{t}}$ ;
	
	Compute gradient $G_i^{(t)} \leftarrow \mathrm{Tr}(M(\lambda^{(t)})^{-1}(x_i x_i^{\top}))$;
	
	Update $\lambda_i^{(t+1)} \leftarrow \frac{\lambda_i^{(t)}\exp{(a_t G_i^{(t)})} }{ \sum_{i=1}^{|\cX_{\sigma}|}\lambda_i^{(t)}\exp{(a_t G_i^{(t)})}}$;
	
	$t \leftarrow t+1$;

	}
	$\lambda^*_{\mathcal{X}_{\sigma}} \leftarrow \lambda^{(t)}$\;

    \Output{ $\lambda^*_{\mathcal{X}_{\sigma}}$}
\end{algorithm}

\section{Discussion on Improving $\alpha$}\label{apx:improve_alpha}

\begin{algorithm}[t!]
	\caption{Approximation algorithm for G-optimal design by the ellipsoid method}\label{alg:G-opt-approx}
	\SetKwInOut{Input}{Input}
	\SetKwInOut{Output}{Output}
	\Input{$d$-set of base arms $[d]$, $n \in \mathbb{Z_+}$, $\tilde{n} \in \mathbb{Z_+}$.}
    
    

     \For{$i =1, \ldots, n$}{
          Choose $\cX_{\sigma_i} \leftarrow$ any $\tilde{n}$-super arms;
          
          Compute $\lambda_i \leftarrow \lambda^*_{\mathcal{X}_{\sigma_i}}$ by Algorithm~\ref{alg:descent} for $\cX_{\sigma_i} $\;
          
          Compute $w_{\lambda_i} \in \mathbb{R}_+^d$ by setting $w_{\lambda_i}= (\sum_{j \in [d]}|M(\lambda_i)^{-1/2}_{e,j}|)_{e \in [d]}  \in \mathbb{R}_+^d$ ;
     }

    Perform the ellipsoid method to solve $\mathrm{LP}_{\mathrm{primal}}$:
\begin{alignat}{4}\label{prob:lp_primal} 
& & \mathrm{LP}_{\mathrm{primal}}:  \ &\text{min.} &\hspace{0.1cm} & \nu \\
&              &    &\text{s.t.} &    &  \nu \geq  \sum_{i \in [n]}h_{i} \sum_{e \in [d]} w_{\lambda_i,e}x_e, \ (\forall x \in \mathcal{X}) \notag\\
  &&&& & h \in \triangle([n]). \notag
\end{alignat}
 
    $h^* \leftarrow$ optimal solution to $\mathrm{LP}_{\mathrm{primal}}$;
    
    $\nu^* \leftarrow$ optimal value of $\mathrm{LP}_{\mathrm{primal}}$
    
    Sample $i^* \in [n]$ from $h^* \in \triangle([n])$;
    
	
	$\alpha \leftarrow  \min\left\{  \frac{\nu^*}{\sqrt{d}}, \min_{i \in [n]}\sqrt{\frac{md}{\xi_{\min}(\tilde{M}(\lambda_i))}} \right\}$  \;
 



    \Output{ $\lambda_{i^*}$ and $\alpha$}
\end{algorithm}


In this section, we further discuss the approximation algorithm for computing $\lambda \in \triangle(\cX)$ and provide Algorithm~\ref{alg:G-opt-approx}, which can be one alternative of Algorithm \ref{alg:compute_dist} in \textsf{PolyALBA}.
 Lemma~\ref{lemma: alpha} indicates that choosing $\mathrm{supp}(\lambda)$ that maximizes $\xi_{\min}(\widetilde{M}(\lambda))$ gives the better bound.
Such a design is so called the \emph{E-optimal design}, i.e., the goal is to minimize the maximum eigenvalue of the error covariance~\cite{pukelsheim2006}.
As noted in the main paper, if we are allowed to pull unit vectors, we have $\xi_{\min}(\widetilde{M}(\lambda)) \geq 1$.
For general cases,
more sophisticated algorithm by the ellipsoid method~\cite{Grotschel1981} is considered in this section
as one alternative of Algorithm~\ref{alg:compute_dist}.
This approach provides the better choice of $\lambda$ (or equivalently $\alpha$) than Algorithm~\ref{alg:compute_dist} in terms of the expectation value.

Let $\triangle(\mathcal{X})_{\mathrm{poly}}$ be the subset of probability distributions over $\mathcal{X}$ with polynomial-size support.
Our task is to find an approximate solution $\tilde{\lambda} \in \triangle(\mathcal{X})_{\mathrm{poly}}$ to the following minmax optimization:
\[
\min_{\lambda \in \triangle(\mathcal{X})}\max_{x \in \mathcal{X}} \| x\|_{M(\lambda)^{-1}}.
\]
We denote the exact G-optimal design by $\lambda^*=\argmin_{\lambda \in \triangle(\mathcal{X})}\max_{x \in \mathcal{X}} \| x\|_{M(\lambda)^{-1}}$.
The above minmax optimization is computationally intractable in combinatorial settings, while many existing methods in linear bandits involved the brute force to solve G-optimal design problem~\cite{Fiez2019,Soare2014,Tao2018}.
To avoid a intractable brute force,
we address a relaxation problem for the minmax optimization, i.e., a randomized mixed strategy for the robust combinatorial optimization.
However, since the ellipsoidal norm $\| x\|_{M(\lambda)^{-1}}$ has the quadratic form, such a relaxation problem is still hard to compute.
To overcome this challenge, we consider a simpler norm instead of the ellipsoidal norm.
In higher level, this idea is similar to that of \citet{Dani2008};
they use skewed octahedron called $\mathrm{ConfidenceBall}_1$ as its confidence region rather than the ellipsoid.
The radius of $\mathrm{ConfidenceBall}_1$ has been set large enough such that it contains the confidence ellipsoid as an inscribed subset. 
Whereas they use 1-norm $\| M(\lambda)^{1/2} x \|_1$ to define $\mathrm{ConfidenceBall}_1$,
however, $\max_{x \in \cX}\| M(\lambda)^{1/2} x \|_1$ is still intractable in the combinatorial action space.
To avoid the computational hardness,
we introduce a linear function $g_\lambda : \{0,1\}^d \rightarrow \mathbb{R}_+$ in order to utilize the underlying combinatorial structure.
We define $w_{\lambda}= (\sum_{j \in [d]}|M(\lambda)^{-1/2}_{i,j}|)_{i \in [d]}  \in \mathbb{R}_+^d$   for $\lambda \in \triangle(\mathcal{X})$.
For $\lambda \in \triangle(\mathcal{X})$,
a linear function  $g_{\lambda}: \{0,1\}^d \rightarrow \mathbb{R}_+$ is represented as $g_{\lambda}(x)=\sum_{e \in [d]} w_{\lambda,e}x_e$.
We shall assume that the Ellipsoid method computes the optimal solution for linear programmings by enough iterations~\cite{Grotschel1981}.
We show that the optimal value $\nu^*$ in Algorithm~\ref{alg:G-opt-approx} gives an upper bound of the confidence ellipsoidal norm.
\begin{lemma}
Let $\mathcal{X}$ be a family of super arms satisfying given constraints such as the matroid, matroid intersection, and $s$-$t$ path.
 Let $\lambda^{h*} \in \triangle(\mathcal{X})_{\mathrm{poly}}$ be an output by Algorithm~\ref{alg:G-opt-approx}.
Let $h^*$ be an optimal solution and $\nu^*$ be the optimal value for $\mathrm{LP}_{\mathrm{primal}}$.
 Then, $\lambda^{h*}$ satisfies 
 \[
 \max_{x \in \mathcal{X}} \E[\| x\|_{M(\lambda^{h*})^{-1}}] \leq \nu^*.
 \]
\end{lemma}

\begin{proof}
By the definition of $w_{\lambda}= (\sum_{j \in [d]}|M(\lambda)^{-1/2}_{i,j}|)_{i \in [d]}$ and definitions of the quadratic norm and 1-norm, we have 
\begin{align}\label{ineq:quadratic_apx}
 \| x\|_{M(\lambda)^{-1}} &=  \|M(\lambda)^{-1/2}x \|_{2}\notag\\
 &\leq||M(\lambda)^{-1/2}x||_{1} \notag \\
 & \leq \sum_{e \in [d]} w_{\lambda,e}x_e \notag\\
&= g_{\lambda}(x) \ (\forall x \in \{0,1\}^d).
\end{align}
Let $y^*=\argmax_{x \in \mathcal{X}}  \| x\|_{M(\lambda)^{-1}}$.
From Eq.~\eqref{ineq:quadratic_apx}, it holds that 
\begin{align}\label{ineq:f_and_g}
\max_{x \in \mathcal{X}}   \| x\|_{M(\lambda)^{-1}}=   \|y^* \|_{M(\lambda)^{-1}} \leq g_{\lambda}(y^*).
\end{align}
Recall that $\lambda^{h*}=\lambda_{i^*}$ where $i^*$ was sampled from $h^* \in \triangle([n])$ in Algorithm~\ref{alg:G-opt-approx}; we have $\E[w_{\lambda^{h*},e}] = \sum_{i \in [n]} h^*_i w_{\lambda_i,e}$ for all $e \in [d]$. 
Thus, for any $x \in \mathcal{X}$, we see that
\[
\E[g_{\lambda^{h*}}(x)] =  \sum_{i \in [n]}  h^*_i g_{\lambda_i}(x).
\]

From the above, we have that
\begin{alignat}{4}\label{ineq:Ef_qg}
\max_{x \in \mathcal{X}} \E[\| x\|_{M(\lambda^{h*})^{-1}}] 
 & \leq \max_{x \in \mathcal{X}} \E[ g_{\lambda^{h*}}(x)] \notag\\
&  =   \max_{x \in \mathcal{X}} \sum_{i \in [n]} h^*_i g_{\lambda_i}(x).
\end{alignat}


Thus, we obtain
\begin{alignat*}{4}
\max_{x \in \mathcal{X}}  \E[ \| x\|_{M(\lambda^{h*})^{-1}}] & \leq   \max_{x \in \mathcal{X}} \sum_{i \in [n]} h^*_i g_{\lambda_i}(x)\\ 
& =\min_{h \in \triangle([n])} \max_{x \in \mathcal{X}} \sum_{i \in [n]} h_{i} g_{\lambda_i}(x) \\
&=\nu^* 
\end{alignat*}
where the first inequality follows by Eq.~\eqref{ineq:Ef_qg} and the equations follow from the fact that $h^*$ is an optimal solution for $\mathrm{LP}_{\mathrm{primal}}$.
\end{proof}

Using the above property, $\alpha$ can be replaced with $\alpha = \min\left\{  \frac{\nu^*}{\sqrt{d}}, \min_{i \in [n]}\sqrt{\frac{md}{\xi_{\min}(\tilde{M}(\lambda_i))}} \right\}$  in the expected sample complexity given in Theorem~\ref{thm:CPE-BL}.
\begin{corollary}~\label{corollary:ellipsoid_method}
Let $\nu^*$ be the optimal value for $\mathrm{LP}_{\mathrm{primal}}$ obtained in Algorithm~\ref{alg:G-opt-approx}.
	With probability at least $1-\delta$, the \textsf{PolyALBA} algorithm (Algorithm \ref{alg:PolyALBA}) will return the best super arm $x^*$ with the expected sample complexity 
		\begin{align*}
		O\Bigg( & \frac{c_0 d (\alpha \sqrt{m} + \alpha^2 )}{\Delta_{d+1}^2} \left(\ln \delta^{-1} + \ln |\mathcal{X}|+\ln\ln \Delta_{d+1}^{-1} \right) 
		\\
		& +\sum_{i=2}^{\left \lfloor  \frac{d}{2} \right \rfloor} \frac{c_0}{\Delta_i^2} (\ln \delta^{-1} + \ln |\mathcal{X}|+\ln\ln \Delta_i^{-1})\Bigg),
	\end{align*}
	where $\alpha =\min\left\{  \frac{\nu^*}{\sqrt{d}},  \min_{i \in [n]}\sqrt{\frac{md}{\xi_{\min}(\tilde{M}(\lambda_i))}} \right\}$.
\end{corollary}
Note that Algorithm~\ref{alg:G-opt-approx} runs in polynomial time as long as $g_{\lambda}(x)$ is linear and not necessarily $g_{\lambda}(x)=\sum_{e \in [d]} w_{\lambda,e}x_e$ defined in this section.

\paragraph{LP-based algorithm for combinatorial robust optimization.}
We briefly explain polynomial-time solvability of $\mathrm{LP}_{\mathrm{primal}}$ in Algorithm~\ref{alg:G-opt-approx} by the Ellipsoid method.
Given a family $\mathcal{X}$ satisfying a combinatorial constraint and $n$-set function $g_{\lambda_1},\ldots,g_{\lambda_n}: 2^{[d]} \rightarrow \mathbb{R}_+$, we describe how to solve the following combinatorial robust optimization:
\begin{align*}
\min_{h \in \triangle([n])} \max_{x \in \mathcal{X}} \sum_{i \in [n]} h_i g_{\lambda_i}(x).
\end{align*}

By von Neumann's minimax theorem,
it holds that 
\begin{align*}
\min_{h \in \triangle([n])} \max_{x \in \mathcal{X}} \sum_{i \in [n]} h_i g_{\lambda_i}(x)= \max_{p \in \triangle(\mathcal{X})} \min_{i \in [n]} \sum_{x \in \mathcal{X}} p_x g_{\lambda_i}(x).
\end{align*}
A \emph{polytope} of $\mathcal{X}$ is defined as $P(\mathcal{X})= \mathrm{conv} \{x \mid x \in \mathcal{X} \}$.
For a vector $x\in \cX$ we define $S^x$ to be the corresponding subset form, i.e.
	$S^x = \{i\in [n] \mid x_i = 1\}$.
The key observation is that for any distribution $p \in \triangle(\mathcal{X})$, we can obtain a point $y \in P(\mathcal{X})$ 
by $y = p^{\top} \cdot x$,
 and thus
	for every $e\in [d]$, $y_e=\sum_{x \in \mathcal{X} \mid e \in S^x} p_x$.
This means that $y_e$ is the marginal probability 
	of seeing an included dimension (a.k.a. a base arm $e$) when selecting vector $x$ (a.k.a. super arm $S^x$) 
	according to distribution $p$.

Then, the optimal value of the above problems is equal to the value of the following LP:
\begin{alignat}{4}\label{prob:lp_dual} 
& &\ \mathrm{LP}_{\mathrm{dual}}:  \ &\text{max.} &\hspace{0.5cm} & s \\
&              &    &\text{s.t.} &    &  s \leq \sum_{e \in [d]}w_{\lambda_i,e} y_e, \ (\forall i \in [n])\notag \\
  &&&& & y \in P(\mathcal{X}). \notag
\end{alignat}
If $\mathcal{X}$ is a matroid, matroid intersection, or the set of $s$-$t$ paths, there exists an efficient \emph{separation oracle} for $P(\mathcal{X})$.
The separation problem for these constraints can be solved in polynomial time as long as  $g_{\lambda_1},\ldots,g_{\lambda_n}$ are linear functions.
Therefore,
due to the theorem of Grötschel et al.~\cite{Grotschel1981},
we can solve the LP in polynomial time in $d$ and $n$.
Note that the Ellipsoid method can find an optimal solution to the dual problem of $\mathrm{LP}_{\mathrm{dual}}$, i.e.,~$\mathrm{LP}_{\mathrm{primal}}$ in \eqref{prob:lp_primal}.
Therefore, we can obtain $h^*=\argmin_{h \in \triangle([n])} \max_{x \in \mathcal{X}} \sum_{i \in [n]} h_i g_{\lambda_i}(x)$.
For the knapsack constraint and the $r (>2)$-matroid intersection constraint, the corresponding separation problems are NP-hard.
Kawase and Sumita~\shortcite{Kawase2019} proposed approximation schemes by solving a separation problem for a relaxation of the polytope, which gives PTAS for the knapsack constraint and $2/(er)$-approximate solution for $r$-matroid intersection constraint.

\section{Equivalence Theorem for Optimal Experimental Design}\label{apx:equivalence_theorem}

We introduce the following equivalence theorem in \citet{kiefer_wolfowitz_1960} adopted our setting of CPE-BL, which will be used in our analysis.
\begin{proposition}[\citet{kiefer_wolfowitz_1960}]\label{proposition:equivalence}
Define $M(\lambda)=\E_{z \sim \lambda}[z z^{\top}]$ for any distribution $\lambda$ supported on $\mathcal{X} \subseteq \mathbb{R}^d$. We consider two extremum problems.

The first is to choose $\lambda$ so that 
\[ (1) \lambda \  \text{maximizes} \ \mathrm{det} \ M(\lambda) \hspace{1cm} (\text{D-optimal design})
\]
The second one is to choose $\lambda$ so that 
\[ (2) \lambda \  \text{minimizes} \ \max_{x \in \mathcal{X}}  x^{\top}M(\lambda)^{-1}x \hspace{1cm} (\text{G-optimal design})
\]
We note that $\E_{x\sim \lambda} [x^{\top} M(\lambda)^{-1} x]$ is $d$,
hence, $\max_{x \in \mathcal{X}}   x^{\top}M(\lambda)^{-1}x  \geq d$, and thus a sufficient condition for $\lambda$ to satisfy $(2)$ is
\[ (3)  \max_{x \in \mathcal{X}}  x^{\top}M(\lambda)^{-1}x =d.
\]
Statements (1), (2) and (3) are equivalent.
\end{proposition}

 \section{Missing proofs}\label{sec:missing_proof}
 
 \subsection{Proof of $\|\hat{\theta}_n -\theta \|_2 \leq \beta_\sigma$ in Section~\ref{sec:GCBPEprocedure}} \label{apx:eq_beta_sigma}
 \begin{proof}
 	 We prove inequality $\|\hat{\theta}_n -\theta \|_2 \leq \beta_\sigma$ in Section~\ref{sec:GCBPEprocedure} using similar techniques in \cite{LinTian2014}.
 	 
 	 Recall that in the \textsf{GCB-PE} algorithm (Algorithm \ref{alg:GCB-PE}), $\hat{\theta}_n$ is the estimate of the environment vector $\theta$ in the $n$-th exploration round. For any $n$,
 	\begin{align*}
 		& \left \| \hat{\theta}_n -\theta \right \|_2 
 		\\
 		= & \left \|  M_\sigma^+ \vec{y}_n - M_\sigma^+ M_\sigma \theta \right \|_2
 		\\
 		= & \left \|  M_\sigma^+ \cdot [M_{x_1} \eta_1; \cdots;  M_{x_{|\sigma|}} \eta_{|\sigma|} ] \right \|_2
 		\\
 		= & \left \| (M_{\sigma}^{\top} M_{\sigma})^{-1} \sum_{i=1}^{|\sigma|} M_{x_i}^{\top} M_{x_i} \eta_i \right \|_2
 		\\
 		\leq & \max_{\eta_1, \cdots, \eta_{|\sigma|}  \in [-1,1]^{d}} \left \|  (M_{\sigma}^{\top} M_{\sigma})^{-1} \sum_{i=1}^{|\sigma|} M_{x_i}^{\top} M_{x_i} \eta_i \right \|_2
 		\\
 		= & \beta_{\sigma}.
 	\end{align*}
 \end{proof}

 \subsection{Proof of Theorem~\ref{thm:GCB_ub}} \label{apx:gcb_ub}
 In order to prove Theorem~\ref{thm:GCB_ub}, we first present the following three lemmas,  Lemma~\ref{lemma:rad}-\ref{lemma:GCB_terminate}.
 
 \begin{lemma} \label{lemma:rad}
	For Algorithm \ref{alg:GCB-PE}, after $n$ exploration rounds,
	$$
	\Pr[\|\theta-\hat{\theta}(n)\|_2 \geq \textup{rad}_n]\leq \frac{\delta}{2n^2}
	$$
\end{lemma}
\begin{proof}
	In Lemma A.3 in \cite{LinTian2014}, let $\gamma=\textup{rad}_n$. Then, we have 
	\begin{align*}
		\Pr[\|\theta-\hat{\theta}(n)\|_2 \geq \textup{rad}_n] \leq & 2 e^2 \textup{exp}  \left \{ -\frac{n}{2 \beta_\sigma^2 } \cdot \frac{2 \beta_\sigma^2 \log( \frac{4n^2e^2}{\delta} ) }{n} \right \}
		\\
		= & \frac{\delta}{2n^2}.
	\end{align*}
\end{proof}

Define the following events 
$$\cE_n:=\{ \forall x \in \cX, |\bar{r}(x, \theta) - \bar{r}(x, \hat{\theta}(n))| < L_p \cdot \textup{rad}_n \}, n \geq 1
$$
$$
\cE:=\bigcap \limits_{n=1}^{\infty} \cE_n. 
$$

\begin{lemma} \label{lemma:gcb_event}
It hols that $
\Pr[ \cE ] \geq 1-\delta.
$
\end{lemma}
\begin{proof}
	From the continuity of the expected reward function, 
	$$
	\bar{r}(x, \theta) - \bar{r}(x, \hat{\theta}(n)) < L_p \cdot \|\theta-\hat{\theta}(n)\|_2 .
	$$
	From Lemma \ref{lemma:rad}, we have that with probability at least $1-\frac{\delta}{2n^2}$,
	$$
	\|\theta-\hat{\theta}(n)\|_2 < \textup{rad}_n .
	$$
	Thus, with probability at least $1-\frac{\delta}{2n^2}$, 
	$$
	\bar{r}(x, \theta) - \bar{r}(x, \hat{\theta}(n)) < L_p \cdot \textup{rad}_n.
	$$
	In other words, 
	$$
	\Pr[ \cE_n ] \geq 1-\frac{\delta}{2n^2} .
	$$
	Thus, we have
	\begin{align*}
	\Pr[\cE] & =  1-\Pr[\bar{\cE}]
	\\
	& \geq 1-\sum \limits_{n=1}^{\infty} \Pr[\bar{\cE}_j]
	\\
	& \geq 1-\sum \limits_{n=1}^{\infty} \frac{\delta}{2n^2}
	\\
	& \geq 1-\delta.
	\end{align*}
\end{proof}

\begin{lemma} \label{lemma:GCB_terminate}
	Suppose that $\cE$ occurs. If $\textup{rad}_n < \frac{\Delta_{\textup{min}}}{4 L_p}$, Algorithm \ref{alg:GCB-PE} will terminate.
\end{lemma}

\begin{proof}
	Suppose that $\cE$ occurs. From the definition of $\cE$, we have
	\begin{align*}
		\bar{r}(\hat{x}, \hat{\theta}(n)) - \bar{r}(\hat{x}^-, \hat{\theta}(n)) & > \bar{r}(\hat{x}, \theta) - \bar{r}(\hat{x}^-, \theta) - 2L_p \cdot \textup{rad}_n
		\\
		& = \Delta_{\textup{min}} - 2L_p \cdot \textup{rad}_n
		\\
		& > 2L_p \cdot \textup{rad}_n
	\end{align*}
	Thus, the stop condition holds, and then Algorithm \ref{alg:GCB-PE} will terminate.
\end{proof}

Now we prove Theorem~\ref{thm:GCB_ub}.
 
\begin{proof}
	First, we prove the correctness of Algorithm \ref{alg:GCB-PE}.
	From the stop condition, we have that when Algorithm \ref{alg:GCB-PE} terminates, for all $x \in \cX \setminus \{ \hat{x} \}$,
	\begin{align*}
		\bar{r}(\hat{x}, \hat{\theta}(n)) - \bar{r}(x, \hat{\theta}(n)) > 2L_p \cdot \textup{rad}_n .
	\end{align*}
	Then, conditioning on $\cE$, when Algorithm \ref{alg:GCB-PE} terminates, for all $x \in \cX \setminus \{ \hat{x} \}$,
	\begin{align*}
	\bar{r}(\hat{x}, \theta) & > \bar{r}(\hat{x}, \hat{\theta}(n)) - L_p \cdot \textup{rad}_n
	\\
	& > \bar{r}(x, \hat{\theta}(n)) + L_p \cdot \textup{rad}_n
	\\
	& > \bar{r}(x, \theta),
	\end{align*}
	which complete the proof of correctness.
	
	Next, we prove the sample complexity of  Algorithm \ref{alg:GCB-PE}.
	Let $N$ denote the total number of the exploration rounds. 
	If $N=1$, Theorem \ref{thm:GCB_ub} trivially holds. In the 
	If $N>1$, from Lemma \ref{lemma:GCB_terminate}, we have that after $N-1$ exploration rounds,
	\begin{align*}
		\sqrt{ \frac{2 \beta_\sigma^2 \log( \frac{4(N-1)^2e^2}{\delta} ) }{N-1} } \geq \frac{\Delta_{\textup{min}}}{4 L_p}
		\\
		N \leq  \frac{32 \beta_\sigma^2 L_p^2}{\Delta_{\textup{min}}^2} \log \left( \frac{4 N^2 e^2}{\delta} \right) + 1 
	\end{align*}
	Let $\tilde{H}:=\frac{ \beta_\sigma^2 L_p^2}{\Delta_{\textup{min}}^2} $. In the following, we prove $N \leq 655 \tilde{H} \log (\frac{\tilde{H}}{\delta})$.
	We can write $N=C \tilde{H} \log (\frac{\tilde{H}}{\delta})$ for some $C>0$.
	In order to prove the theorem, it suffices to prove $C \leq 655$.
	Suppose, on the contrary, that $C > 655$. Then, we have
	\begin{align*}
		N  \leq & 32 \tilde{H} \log \left( \frac{4 N^2 e^2}{\delta} \right) + 1
		\\
		 =	&64 \tilde{H} \log \left( \frac{2 e C \tilde{H} \log \frac{\tilde{H}}{\delta} }{\delta} \right) + 1
		\\
		 \leq &	64 \tilde{H} \log(2 e C) + 64 \tilde{H} \log \frac{\tilde{H}}{\delta} + 64 \tilde{H} \log \left( \log \frac{\tilde{H}}{\delta} \right)  \\& + \tilde{H} \log\frac{\tilde{H}}{\delta}
		\\
		 \leq &	64 \tilde{H} \log(2 e C) + 129 \tilde{H} \log \frac{\tilde{H}}{\delta}
		\\
		 < & C \tilde{H} \log \frac{\tilde{H}}{\delta}
		\\
		 = & N,
	\end{align*}
	which makes a contradiction.
	Thus, 
	$$
	N \leq 655 \tilde{H} \log (\frac{\tilde{H}}{\delta}).
	$$
	Since an exploration round contains $|\sigma| \leq n$ actions, the total number of samples
	$$
	T=|\sigma| \cdot N \leq  \frac{655 \beta_\sigma^2 L_p^2}{\Delta_{\textup{min}}^2} \log \left(\frac{ \beta_\sigma^2 L_p^2}{\Delta_{\textup{min}}^2 \delta}\right).
	$$
\end{proof}
 

\subsection{ Proof of Lemma~\ref{lemma: alpha}}
\begin{proof}
	Recall that in Algorithm~\ref{alg:compute_dist}, we choose $d$ super arms $\mathcal{X}_{\sigma}=\{x_1, \dots, x_d\}$ from $\mathcal{X}$, such that  $\mathrm{rank}(X)=d$ where $X=(x_1, \dots, x_d)$.
	Then, for any super arm $z \in \mathcal{X}$, $z$ can be written as a linear combination of $x_1,x_2, \dots ,x_d$, i.e., 
	$$
	z=Xw,
	$$
	where $w\in \mathbb{R}^{d}$ is the vector of coefficients. Let $\xi_{\textup{min}}(A)$ denote the smallest eigenvalue of matrix $A$.
	Then, we have	
	\begin{align*}
	\sum_{k=1}^{d}|w_k| \leq & \sqrt{d \left(\sum_{k=1}^{d} w_k^2 \right)} 
	\\
	= & \sqrt{d w^{\top} w}
	\\
	\leq & \sqrt{d w^{\top} X^{\top} X w} \cdot \max_{w'\in \mathbb{R}^{d}} \sqrt{\frac{w'^{\top} w'}{w'^{\top} X^{\top} X w'} }
	\\
	\leq & \sqrt{d w^{\top} X^{\top} X w} \cdot  \sqrt{\frac{1}{\xi_{\textup{min}}(X^{\top} X)} }
	\\
	=  &  \sqrt{\frac{d}{\xi_{\textup{min}}(X^{\top} X)} } \|z\|_2
	\\
	\leq &  \sqrt{\frac{md}{\xi_{\textup{min}}(X^{\top} X)} } 
	\\
	= & \sqrt{\frac{md}{\xi_{\textup{min}}(X X^{\top})} } 
	\\
	= & \alpha.
	\end{align*}
	
	Note that in Algorithm~\ref{alg:compute_dist}, we compute $\lambda^*_{\mathcal{X}_{\sigma}} = \argmin_{\lambda \in \triangle( {\mathcal{X}_{\sigma})}} \max_{x \in {\mathcal{X}_{\sigma}}} x^{\top}M(\lambda)^{-1}x$ by the entropic mirror descent (Algorithm~\ref{alg:descent} in Appendix~\ref{apx:G-opt}). $\lambda^*_{\mathcal{X}_{\sigma}}$ is the solution to Proposition \ref{proposition:equivalence} and satisfies $ \max_{x \in {\mathcal{X}_{\sigma}}} x^{\top}M(\lambda^*_{\mathcal{X}_{\sigma}})^{-1}x= d$. Thus, $ \max_{x \in {\mathcal{X}_{\sigma}}} \|x\|_{M(\lambda^*_{\mathcal{X}_{\sigma}})^{-1}} = \sqrt{d}$.
	Thus, for any $z \in \cX$ we have
	\begin{align*}
	\|z\|_{M(\lambda^*_{\mathcal{X}_{\sigma}})^{-1}}= & \|w_1 x_1 + \dots + w_d x_d\|_{M(\lambda^*_{\mathcal{X}_{\sigma}})^{-1}}
	\\
	\leq & |w_1| \cdot \|x_1\|_{M(\lambda^*_{\mathcal{X}_{\sigma}})^{-1}} + \dots \\& + |w_d| \cdot \|x_d\|_{M(\lambda^*_{\mathcal{X}_{\sigma}})^{-1}}
	\\
	\leq & |w_1| \cdot \sqrt{d} + \dots  + |w_d| \cdot \sqrt{d} 
	\\
	= &  (|w_1|   + \dots + |w_d| ) \sqrt{d}
	\\
	\leq & \alpha \sqrt{d}.
	\end{align*}

\end{proof}


\subsection{Proof of Theorem~\ref{thm:CPE-BL}}\label{apx:theorem2}
\subsubsection{Technical lemmas for Theorem~\ref{thm:CPE-BL}}
 \begin{lemma} \label{lemma:rad_ell_varepsilon}
	When $n \geq c_0 \ell(\varepsilon)  \ln (\frac{5|\mathcal{X}|}{\delta_r}) $ where $\varepsilon \leq 3$, we have 
	$$
	\Pr[ |x^{\top} \theta - x^{\top} \hat{\theta} | \leq  \varepsilon , \forall x \in \cX] \geq 1-\delta .
	$$
\end{lemma}

\begin{proof}
We introduce the high probability bound for the estimator $\hat{\theta}$ as follows.
 \begin{proposition}[Lemma 10 ~\citet{Tao2018}]\label{proposi:Tao_Lemma10}

Let $c_0=\max\{4L^2, 3\}$.
Let $n \geq \ell \ln(5/8)$ where $\ell \geq d$. For any fixed $x \in \cX$, with probability at least $1-\delta$, we have
\begin{align*}
	& \qquad |x^{\top}(\theta-\hat{\theta})| \leq  
	\\  & \sqrt{ \frac{ 2\|x\|_2^2 \!+\! 2\sqrt{d}\|x\|_2\|x\|_{M(\lambda)^{-1}} \!\!+\! (4 \!+\! 2\sqrt{d/\ell}) \|x\|^2_{M(\lambda)^{-1}} }{\ell}  }.
\end{align*}
    
\end{proposition}

	Since $\|x\|_2 \leq \sqrt{m}$ and $\|x\|_{M(\lambda)^{-1}} \leq \alpha \sqrt{d}$ from Lemma~\ref{lemma: alpha}, applying Proposition~\ref{proposi:Tao_Lemma10} for every super arm in $\cX$ and via a union bound, we have that 
	 when $n \geq c_0 \ell  \ln (\frac{5|\mathcal{X}|}{\delta_r})$ where $\ell \geq d$, 
	 \begin{align*}
	 	& \Pr \Bigg[ |x^{\top} \theta - x^{\top} \hat{\theta} | \leq \\ &
	 	\sqrt{ \frac{2m+ 2\alpha\sqrt{m} d + (4+2 \sqrt{d/\ell})\alpha^2d }{\ell} }
	 	, \forall x \in \cX \Bigg] \geq 1-\delta .
	 \end{align*}

	 Setting $\ell$ as $\ell(\varepsilon) := \frac{2m + 2 \alpha \sqrt{m} d + 4 \alpha^2 d + \alpha \varepsilon d}{\varepsilon^2}$, we have that with probability at least $1-\delta$, $\forall x \in \cX$,
	 \begin{align*}
	 	|x^{\top} \theta - x^{\top} \hat{\theta} |  \leq  &
	 	\sqrt{ \frac{2m+ 2\alpha\sqrt{m} d + (4+2 \sqrt{d/\ell})\alpha^2d }{\ell} }
	 	\\
	 	 = & \Bigg( \frac{2m+ 2\alpha\sqrt{m} d + 4 \alpha^2 d}{2m + 2 \alpha \sqrt{m} d + 4 \alpha^2 d + \alpha \varepsilon d} \varepsilon^2 \\&+ \frac {2\alpha^2 d^{\frac{3}{2}} }{(2m + 2 \alpha \sqrt{m} d + 4 \alpha^2 d + \alpha \varepsilon d)^{\frac{3}{2}}} \varepsilon^3 \Bigg)^{\frac{1}{2}}
	 	\\
	 	 \leq & \varepsilon  \Bigg( 1- \frac{\alpha \varepsilon d}{2m + 2 \alpha \sqrt{m} d + 4 \alpha^2 d + \alpha \varepsilon d}  \\&+   \frac{2 \alpha \sqrt{d} }{ \sqrt{2m + 2 \alpha \sqrt{m} d + 4 \alpha^2 d + \alpha \varepsilon d} }  \cdot \\& \qquad \frac {\alpha \varepsilon d }{(2m + 2 \alpha \sqrt{m} d +  4 \alpha^2 d + \alpha \varepsilon d) } \Bigg)^{\frac{1}{2}}
	 	\\
	 	 \leq & \varepsilon,
	 \end{align*}
	 which completes the proof.
\end{proof}

Next, we show the sample complexity bound for the epoch $q=0$.
\begin{lemma} \label{lemma:samples_epoch_0}
	With probability at least $1-\delta_0$, the first epoch $q=0$ in Algorithm \ref{alg:PolyALBA} satisfies the following properties: (i)  epoch $q=0$ ends with $x^* \in S_{1}$; and (ii) the sample complexity is bounded by 
	$$
	O \left(\frac{c_0 (\alpha \sqrt{m} d + \alpha^2 d)}{\Delta_{d+1}^2} \left(\ln \delta^{-1} + \ln |\mathcal{X}|+\ln\ln \Delta_{d+1}^{-1} \right) \right) .
	$$
\end{lemma}


\begin{proof}
	For round $r$ of the first epoch, define event $\cF_r:=\{  |x^{\top} \theta - x^{\top} \hat{\theta}_r | \leq  \frac{\varepsilon}{2} , \forall x \in \cX \}$. Applying Lemma \ref{lemma:rad_ell_varepsilon}, we have 
	$
	\Pr [ \cF_r ] \leq 1- \delta_r
	$.
	Define event $\cF:= \bigcap_{r=1}^{\infty} \cF_r $. By a union bound, we have 
	$
	\Pr [ \cF ] \geq 1- \sum_{r=1}^{\infty} \delta_r = 1- \sum_{r=1}^{\infty} \frac{6}{\pi^2} \frac{\delta_q}{r^2} \geq 1-\delta_q
	$.
	We condition the remaining proof on event $\cF$.
	
	(i) First, we show that the first epoch $q=0$ will end with  $x^* \in S_{1}$. Let $x_1, x_2, \dots$ denote the super arms ranked by $x^{\top} \theta$ for $x \in \cX$ (i.e., $x_1^{\top}\theta \geq x_2^{\top}\theta, \ldots$), and we use $x^*$ and $x_1$ interchangeably.
	
	For any round $r \geq 1$, $x_1^{\top} \hat{\theta}_r \geq  x_1^{\top} \theta - \frac{\varepsilon_r}{2} \geq  \hat{x}_1^{\top} \theta - \frac{\varepsilon_r}{2} \geq \hat{x}_1^{\top} \hat{\theta}_r - \varepsilon_r$. Rearranging the terms, we have $\hat{x}_1^{\top} \hat{\theta}_r-x_1^{\top} \hat{\theta}_r \leq \varepsilon_r$, which implies that $x_1$ will never be discarded. Thus, when the first epoch $q=0$ ends, $x_1 \in S_{1}$.
	
	Let $r^*$ be the smallest round such that $\varepsilon_{r^*} < \frac{\Delta_{d+1}}{2} $. In round $r^*$, for $x_{i}$ s.t.  $i \geq d+1$,
	$
		x_1^{\top} \hat{\theta}_{r^*} - x_{i}^{\top} \hat{\theta}_{r^*} \geq (x_1^{\top} \theta - \frac{\varepsilon_{r^*}}{2})-( x_{i}^{\top} \theta + \frac{\varepsilon_{r^*}}{2}) \geq \Delta_i - \varepsilon_{r^*} \geq \Delta_{d+1} - \varepsilon_{r^*} >\varepsilon_{r^*}
	$. Then, the first epoch $q=0$ will end.
	
	(ii) Since the first epoch $q=0$ will end  in (or before) round $r^*$, which is the smallest round such that $\varepsilon_{r^*} < \frac{\Delta_{d+1}}{2} $, then the sample complexity of the first epoch $q=0$ is bounded by 
	\begin{align*}
		& O \left (\frac{ c_0 (\alpha \sqrt{m} d +  \alpha^2 d ) }{\varepsilon_{r^*}^2} \ln \left(\frac{|\mathcal{X}|}{\delta_{r^*}} \right) \right ) 
		\\
		= &
		O \left(\frac{c_0 (\alpha \sqrt{m} d + \alpha^2 d)}{\Delta_{d+1}^2} \left(\ln \delta^{-1} + \ln |\mathcal{X}|+\ln\ln \Delta_{d+1}^{-1} \right) \right) .
	\end{align*}
\end{proof}

\subsubsection{Proof of Theorem~\ref{thm:CPE-BL}}

\begin{proof}
	Define $q^*=\left \lfloor  \log_2d \right \rfloor$. 
	For epoch $q=0$, applying Lemma \ref{lemma:samples_epoch_0}, the sample complexity is bounded by 
	$$O \left(\frac{c_0 (\alpha \sqrt{m} d + \alpha^2 d)}{\Delta_{d+1}^2} \left(\ln \delta^{-1} + \ln |\mathcal{X}|+\ln\ln \Delta_{d+1}^{-1} \right) \right) .
	$$  
	For epoch $q\geq 1$, the \textsf{PolyALBA} algorithm directly calls
	subroutine \textsf{ALBA}.
	Applying Lemma 17 in \cite{Tao2018}, we can bound the sample complexity for epoch $q\geq 1$ by
	\begin{align*}
	& \sum_{q=1}^{q^*} O \left( \frac{c_0 \left \lfloor  \frac{d}{2^{q-1}} \right \rfloor}{\Delta_{\left \lfloor  \frac{d}{2^q} \right \rfloor + 1}^2} \left(\ln \delta^{-1} + \ln |\mathcal{X}|+\ln\ln \Delta_{\left \lfloor  \frac{d}{2^q} \right \rfloor + 1}^{-1} \right)\right)
	\\
	 = & \sum_{q=1}^{q^*} O \left( \frac{c_0 \left \lfloor  \frac{d}{2^{q}} \right \rfloor}{\Delta_{\left \lfloor  \frac{d}{2^q} \right \rfloor + 1}^2} \left(\ln \delta^{-1} + \ln |\mathcal{X}|+\ln\ln \Delta_{\left \lfloor  \frac{d}{2^q} \right \rfloor + 1}^{-1} \right)\right)	
	 \\
	 = & \sum_{q=1}^{q^*} O \Bigg( \frac{c_0 (\left \lfloor  \frac{d}{2^{q}} \right \rfloor - \left \lfloor  \frac{d}{2^{q+1}} \right \rfloor ) }{\Delta_{\left \lfloor  \frac{d}{2^q} \right \rfloor + 1}^2} \cdot\\& \qquad\qquad \left(\ln \delta^{-1} + \ln |\mathcal{X}|+\ln\ln \Delta_{\left \lfloor  \frac{d}{2^q} \right \rfloor + 1}^{-1} \right)\Bigg)	
	 \\
	 = & \sum_{q=1}^{q^*-1} \sum_{\left \lfloor  \frac{d}{2^{q+1}} \right \rfloor+ 1 }^{\left \lfloor  \frac{d}{2^{q}} \right \rfloor } O \left(\frac{c_0  }{\Delta_i^2} \left(\ln \delta^{-1} + \ln |\mathcal{X}|+\ln\ln \Delta_{i}^{-1} \right)\right)	
	 \\
	 = & O \left( \sum_{i=2}^{\left \lfloor  \frac{d}{2} \right \rfloor} \frac{c_0}{\Delta_i^2} \left(\ln \delta^{-1} + \ln |\mathcal{X}|+\ln\ln \Delta_i^{-1}\right) \right ).
	\end{align*}

	Summing the sample complexity for epoch $q=0$ and $q \geq 1$, we obtain the theorem.
\end{proof}

\section{Experiments on the Top-$k$ Instances} \label{apx:top_k_experiments}

\begin{figure}[t]
	\centering
	\subfigure[The number of samples and running time with $(d,k)=(8,3)$. We vary the minimum gap $\Delta_{\min}$ from 0.1 to 1.0.
			Each point is an average over 10 realizations.]{\label{fig:minimal_gap}
		\includegraphics[width=0.47\columnwidth]{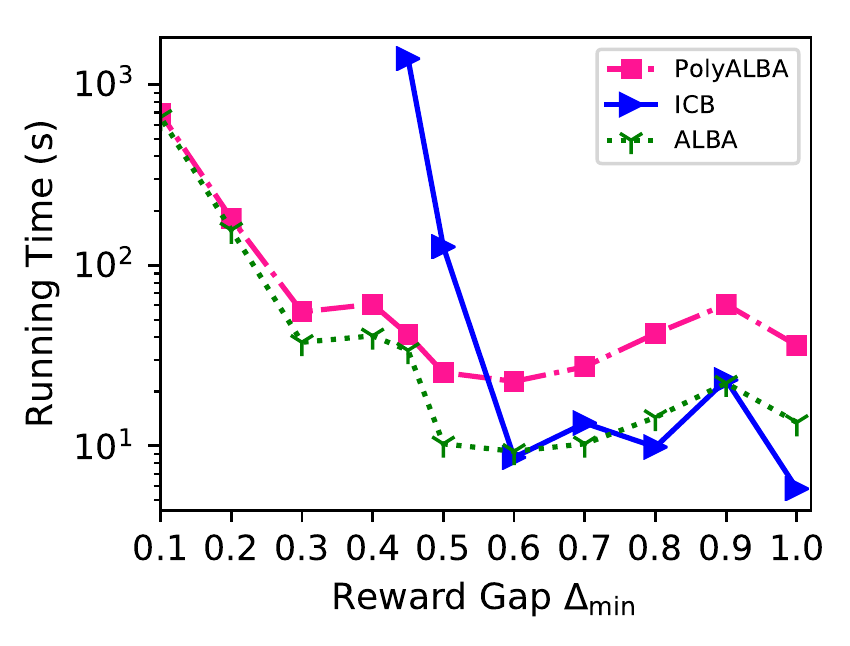}
		\includegraphics[width=0.47\columnwidth]{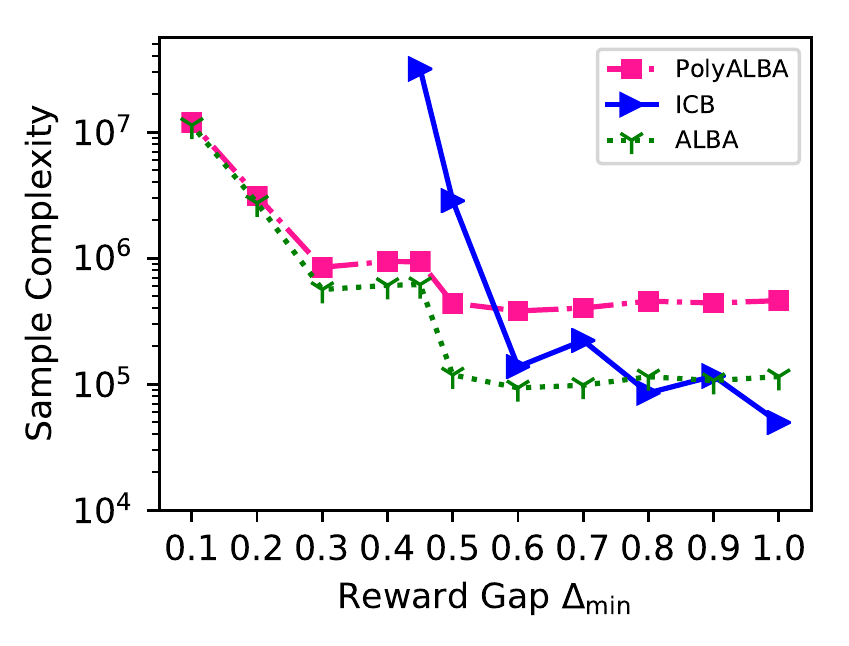}

	}
	\quad
	\subfigure[The number of samples and running time with $k=3$, $d=8, 10, 12, 14, 16$, and $\Delta_{\min}=0.5$.
			Each point is an average over 10 realizations.]{\label{fig:scale}
		\includegraphics[width=0.47\columnwidth]{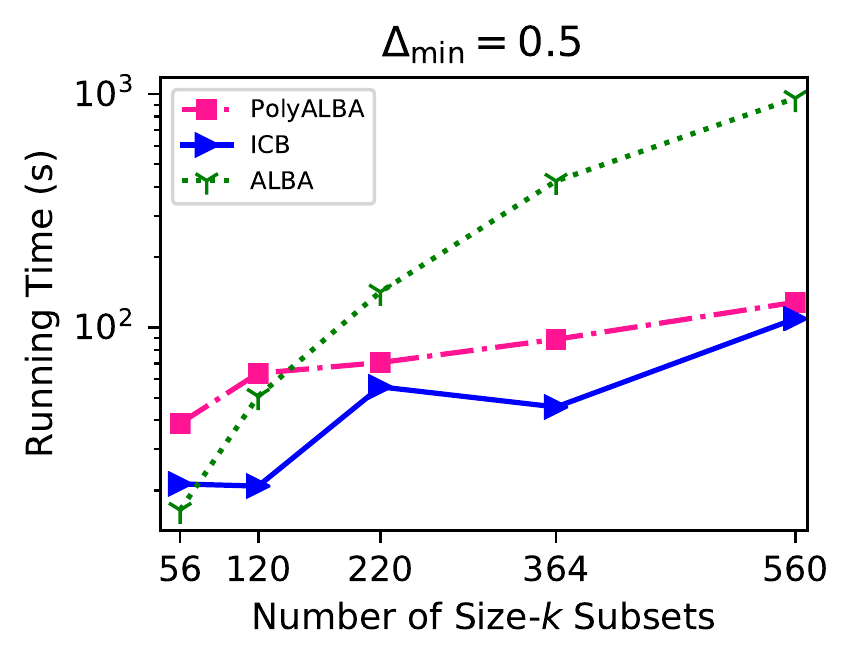}
		\includegraphics[width=0.47\columnwidth]{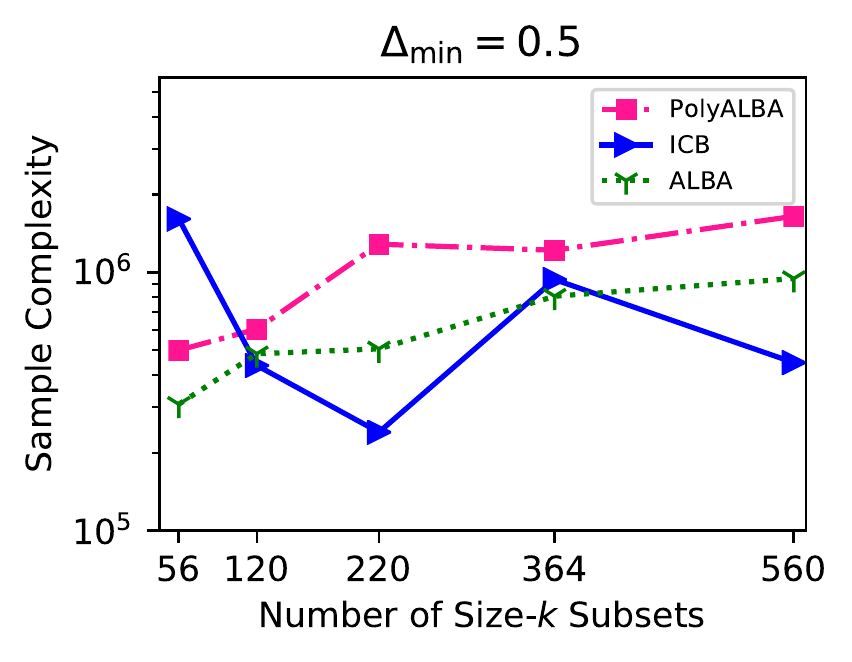}
	}
	\caption{Experimental results of running time and sample complexity  for full-bandit top-$k$ instances.}
	\label{fig:topk_experiments}
\end{figure}

The main purpose of the experiments in this section is to see the dependence of the performance on the minimum gap $\Delta_{\min}$.
We empirically demonstrate that \textsf{PolyALBA} is robust across different $\Delta_{\min}$ settings and runs much faster than existing BAI-LB algorithms
as reported in Figure~\ref{fig:topk_experiments}.

\paragraph{Experimental settings.}
As a polynomial-time baseline algorithm, we implement \textsf{ICB}~\cite{Kuroki+19}, whose sample complexity heavily depends on $\Delta_{\min}^{-2}$.
We evaluate \textsf{PolyALBA}, \textsf{ALBA}, and \textsf{ICB} 
for the top-$k$ case of CPE-BL  on Intel Xeon E5-2640 v3 CPU at 2.60GHz with 132GB RAM.
We set the expected rewards
for the top-$k$ base arms uniformly at random from $[0.5 ,1.0]$.
Let $\theta_{{\min}\text{-}k}$
be the minimum expected reward
in the top-$k$ base arms.
We set the expected reward of
the top $(k+1)$-th base arm to
$\theta_{{\min}\text{-}k}-\Delta_{\min}$
for the predetermined parameter $\Delta_{\min} \in [0.1, 1.0]$.
Then, 
we generate the expected rewards of
the rest of base arms by uniform samples from $[-1.0, \theta_{{\min}\text{-}k}-\Delta_{\min}]$
so that expected rewards of the best super-arm is
larger than those of the rest of super arms by at least $\Delta_{\min}$.
We set the additive noise distribution ${\cal N}(0,1)$.
In all instances we set $\delta=0.05$.
In order to perform the exponential-time algorithms \textsf{ALBA} in reasonable time, we run the experiments with $d=8$, $k=3$, and thus $|\cX|=56$.
Note that since \textsf{RAGE} is already prohibitive in these instances due to its memory error, we exclude it here.
The result is shown in Figure~\ref{fig:minimal_gap}.
In the second experiment,
we evaluate \textsf{PolyALBA}, \textsf{ALBA}, and \textsf{ICB} on the synthetic instances with varying $|\cX|$ and fixed $\Delta_{\min}$.
To perform \textsf{ICB} with a reasonable sample complexity, we set $\Delta_{\min}=0.5$ in all instances.
We vary the number of base arms $d=[8, 10, 12, 14, 16]$ while $k=3$ is fixed. Thus, $|\cX| \in [56, 120, 220, 364, 560]$ in this experiment.
The result is shown in Figure~\ref{fig:scale}.

\paragraph{Results.}
As can be seen in Figure~\ref{fig:minimal_gap},
\textsf{PolyALBA} performs well in all instances, while a polynomial-time baseline \textsf{ICB} is sensitive to the value of $\Delta_{\min}$, which matches our theoretical analysis.
Indeed, \textsf{ICB} cannot stop even after several days because it requires at least more than $10^9$ samples when $\Delta_{\min}=0.4$.
On the other hand,  \textsf{PolyALBA} is still competitive with the exponential-time baseline \textsf{ALBA}.
Figure~\ref{fig:scale} demonstrates that
\textsf{ALBA} is too slow in larger sized instances;
we report that \textsf{ALBA} cannot work in the large instance with $d=18$ and $|\cX|=816$ due to its memory issue in our environment.
From the results,
\textsf{PolyALBA} is the only practical algorithm that is efficient in terms of both time complexity and sample complexity.
We validate that the sample  complexity of \textsf{PolyALBA} is robust against the minimum gap, which well suits the real-world applications.

\end{document}